\documentclass{article}
\pdfoutput=1

\usepackage[final,nonatbib]{neurips_2020}

\usepackage[utf8]{inputenc} 
\usepackage[T1]{fontenc}    
\usepackage{hyperref}       
\usepackage{url}            
\usepackage{booktabs}       
\usepackage{amsfonts}       
\usepackage{nicefrac}       
\usepackage{microtype}      

\usepackage{fancyhdr,amsmath,amsthm,xspace,algcompatible, enumerate, paralist, mathtools, tabularx, hhline,xcolor,colortbl,tabu,forest,xcolor,array}
\usepackage[noline, ruled,boxed]{algorithm2e}
\usepackage[capitalise]{cleveref}
\usepackage[export]{adjustbox}
\usepackage{pgf,tikz}
\usepackage{graphicx}
\usepackage{stmaryrd}
\usepackage{subcaption}
\usepackage{bm}
\usepackage{mathrsfs}
\usepackage{type1cm}
\usepackage{csquotes}
\usepackage{amssymb}
\usepackage{comment}
\usepackage{multirow}
\usepackage{pifont}
\usepackage{verbatim}
\usepackage{enumitem}
\usepackage{float}
\usepackage{wrapfig}
\usepackage{thmtools}
\usepackage{dsfont}
\usepackage{thm-restate}
\usepackage{soul}
\usepackage{sansmath}

\tikzstyle{mybox} = [draw=gray, fill=gray!20, very thick,
rectangle, rounded corners, inner xsep=3pt, inner ysep=7pt]

\usepackage{pgfplots}
\DeclareUnicodeCharacter{2212}{−}
\usepgfplotslibrary{groupplots,dateplot}
\usetikzlibrary{patterns,shapes.arrows}
\pgfplotsset{compat=newest}

\usetikzlibrary{positioning,arrows,shapes,shapes.arrows,arrows.meta,trees,shapes.misc,shapes.geometric,decorations.pathreplacing,automata}
\tikzset{%
  >={Latex[width=1.5mm,length=2mm]},
  vertex/.style={draw,circle,inner sep=0mm,semithick,minimum width=4mm},
  point/.style = {circle, draw, inner sep=0.04cm,fill,node contents={}},
  uvertex/.style={draw,circle,dashed,inner sep=0mm,semithick,minimum width=4mm},
  bidir/.style={<->,dashed, line width=0.25mm},
  dir/.style={->, line width=0.25mm},
  regime/.style={shape=rectangle,fill=black,inner sep=0pt,minimum size=3pt,draw},
  node distance=1cm,
  font=\scriptsize\sffamily\sansmath
}
\definecolor{betterred}{RGB}{228,26,28}
\definecolor{betterblue}{RGB}{55,126,184}

\pgfdeclarelayer{back}
\pgfsetlayers{back,main}

\makeatletter
\pgfkeys{%
  /tikz/on layer/.code={
    \def\tikz@path@do@at@end{\endpgfonlayer\endgroup\tikz@path@do@at@end}%
    \pgfonlayer{#1}\begingroup%
  }%
}
\makeatother

\newcommand{\convexpath}[2]{
  [   
  create hullcoords/.code={
    \global\edef\namelist{#1}
    \foreach [count=\counter] \nodename in \namelist {
      \global\edef\numberofnodes{\counter}
      \coordinate (hullcoord\counter) at (\nodename);
    }
    \coordinate (hullcoord0) at (hullcoord\numberofnodes);
    \pgfmathtruncatemacro\lastnumber{\numberofnodes+1}
    \coordinate (hullcoord\lastnumber) at (hullcoord1);
  },
  create hullcoords
  ]
  ($(hullcoord1)!#2!-90:(hullcoord0)$)
  \foreach [
  evaluate=\currentnode as \previousnode using \currentnode-1,
  evaluate=\currentnode as \nextnode using \currentnode+1
  ] \currentnode in {1,...,\numberofnodes} {
    let \p1 = ($(hullcoord\currentnode) - (hullcoord\previousnode)$),
    \n1 = {atan2(\y1,\x1) + 90},
    \p2 = ($(hullcoord\nextnode) - (hullcoord\currentnode)$),
    \n2 = {atan2(\y2,\x2) + 90},
    \n{delta} = {Mod(\n2-\n1,360) - 360}
    in 
    {arc [start angle=\n1, delta angle=\n{delta}, radius=#2]}
    -- ($(hullcoord\nextnode)!#2!-90:(hullcoord\currentnode)$) 
  }
}

\algdef{SE}[SUBALG]{INDENT}{ENDINDENT}{}{\algorithmicend\ }%
\algtext*{INDENT}
\algtext*{ENDINDENT}

\newcommand{\ci}{\perp \!\!\!\! \perp }
\newcommand{\mli}[1]{\mathit{#1}}

\Crefname{equation}{Eq.}{Eqs.}
\Crefname{figure}{Fig.}{Figs.}
\Crefname{tabular}{Tab.}{Tabs.}
\Crefname{theorem}{Thm.}{Thms.}
\Crefname{lemma}{Lem.}{Lems.}
\Crefname{proposition}{Prop.}{Props.}
\Crefname{definition}{Def.}{Defs.}
\Crefname{algorithm}{Alg.}{Algs.}
\Crefname{corollary}{Corol.}{Corol.}
\Crefname{section}{Sec.}{Sec.}

\newtheorem{lemma}{Lemma}

\newtheorem{corollary}{Corollary}

\theoremstyle{definition}
\newtheorem{definition}{Definition}

\newcommand{\braces}[2][]{#1\{#2 #1\}}

\newcommand{\angles}[2][]{#1\langle#2 #1\rangle}

\newcommand{\set}[2][]{\braces[#1]{#2}}

\newcommand{\tuple}[2][]{\angles[#1]{#2}}

\makeatletter
\newcommand{\xdashleftrightarrow}[2][]{\ext@arrow 3359\leftrightarrowfill@@{#1}{#2}}
\def\rightarrowfill@@{\arrowfill@@\relax\relbar\rightarrow}
\def\leftarrowfill@@{\arrowfill@@\leftarrow\relbar\relax}
\def\leftrightarrowfill@@{\arrowfill@@\leftarrow\relbar\rightarrow}
\def\arrowfill@@#1#2#3#4{%
  $\m@th\thickmuskip0mu\medmuskip\thickmuskip\thinmuskip\thickmuskip
   \relax#4#1
   \xleaders\hbox{$#4#2$}\hfill
   #3$%
}
\makeatother

\newcommand{\E}{\mathbb{E}}

\newcommand{\VV}{\bm{V}}

\newcommand{\XX}{\bm{X}}
\newcommand{\xx}{\bm{x}}

\newcommand{\ZZ}{\bm{Z}}
\newcommand{\zz}{\bm{z}}

\newcommand{\sss}{\bm{s}}

\newcommand{\D}{\mathscr{D}}

\newcommand{\G}{\mathcal{G}}

\newcommand{\PP}{\mathscr{P}}

\newcommand{\FF}{\mathcal{F}}

\newcommand{\Pa}{\mli{Pa}}
\newcommand{\pa}{\mli{pa}}

\newcommand{\De}{\mli{De}}
\newcommand{\An}{\mli{An}}
\newcommand{\an}{\mli{an}}
\newcommand{\de}{\mli{de}}

\newcommand{\Ch}{\mli{Ch}}
\newcommand{\ch}{\mli{ch}}

\newcommand{\LL}{\bm{L}}
\newcommand{\RR}{\bm{R}}

\newcommand{\doo}{\text{do}}

\def\*#1{\bm{#1}}
\def\1#1{\mathcal{#1}}
\def\2#1{\mathscr{#1}}

\title{Causal Imitation Learning\\with Unobserved Confounders}

\author{
  Junzhe Zhang\\
  Columbia University\\
  \texttt{junzhez@cs.columbia.edu}
  \And 
  Daniel Kumor\\
  Purdue University\\
  \texttt{dkumor@purdue.edu}
  \And
  Elias Bareinboim\\
  Columbia University\\
  \texttt{eb@cs.columbia.edu}
}

\begin{document}


\maketitle

\begin{abstract}
One of the common ways children learn is by mimicking adults. Imitation learning focuses on learning policies with suitable performance from demonstrations generated by an expert, with an unspecified performance measure, and unobserved reward signal. Popular methods for imitation learning start by either directly mimicking the behavior policy of an expert (\emph{behavior cloning}) or by learning a reward function that prioritizes observed expert trajectories (\emph{inverse reinforcement learning}). However, these methods rely on the assumption that covariates used by the expert to determine her/his actions are fully observed. In this paper, we relax this assumption and study imitation learning when sensory inputs of the learner and the expert differ. First, we provide a non-parametric, graphical criterion that is complete (both necessary and sufficient) for determining the feasibility of imitation from the combinations of demonstration data and qualitative assumptions about the underlying environment, represented in the form of a causal model. We then show that when such a criterion does not hold, imitation could still be feasible by exploiting quantitative knowledge of the expert trajectories. Finally, we develop an efficient procedure for learning the imitating policy from experts' trajectories. 
\end{abstract}
\section{Introduction}\label{sec1}
A unifying theme of Artificial Intelligence is to learn a policy from observations in an unknown environment such that a suitable level of performance is achieved \cite[Ch.~1.1]{russell2002artificial}. Operationally, a policy is a decision rule that determines an action based on a certain set of covariates; observations are possibly generated by a human demonstrator following a different \emph{behavior policy}. The task of evaluating policies from a combination of observational data and  assumptions about the underlying environment has been studied in the literature of causal inference \cite{pearl:2k} and reinforcement learning \cite{sutton1998reinforcement}. Several criteria, algorithms, and estimation methods have been developed to solve this problem \cite{pearl:2k,spirtes2000causation,bareinboim:pea:16-r450,correa2020calculus,shpitser2018identification,rosenbaum:rub83,watkins1992q}. In many applications, it is not clear which performance measure the demonstrator is (possibly subconsciously) optimizing. That is, the reward signal is not labeled and accessible in the observed expert's trajectories. In such settings, the performance of candidate policies is not uniquely discernible from the observational data due to latent outcomes, even when infinitely many samples are gathered, complicating efforts to learn policy with satisfactory performance.

An alternative approach used to circumvent this issue is to find a policy that mimics a demonstrator's behavior, which leads to the \emph{imitation learning} paradigm \cite{argall2009survey,billard2008survey,hussein2017imitation,osa2018algorithmic}. The expectation (or rather hope) is that if the demonstrations are generated by an expert with near-optimal reward, the performance of the imitator would also be satisfactory. Current methods of imitation learning can be categorized into \textit{behavior cloning} \cite{widrow1964pattern,pomerleau1989alvinn,muller2006off,mulling2013learning,mahler2017learning} and \textit{inverse reinforcement learning} \cite{ng2000algorithms,abbeel2004apprenticeship,syed2008game,ziebart2008maximum}. The former focuses on learning a nominal expert policy that approximates the conditional distribution mapping observed input covariates of the behavior policy to the action domain. The latter attempts to learn a reward function that prioritizes observed behaviors of the expert; reinforcement learning methods are then applied using the learned reward function to obtain a nominal policy. However, both families of methods rely on the assumption that the expert's input observations match those available to the imitator. When unobserved covariates exist, however, naively imitating the nominal expert policy does not necessarily lead to a satisfactory performance, even when the expert him or herself behaves optimally.

\begin{figure}[t]
\hfill%
\begin{subfigure}{0.2\linewidth}\centering
  \begin{tikzpicture}
      \def\outerr{3}
      \def\innerr{2.7}

      \node[vertex] (X) at (0, 0) {$X$};
      \node[vertex] (Z) at (1, 1.5) {$Z$};
      \node[uvertex] (L) at (1, 0.5) {$L$};
      \node[uvertex] (Y) at (2, 0) {$Y$};

      \draw[dir] (L) -- (Y);
      \draw[dir] (Z) -- (L);
      \draw[dir] (L) -- (X);
      \draw[dir] (Z) to (Y);
      \draw[dir] (X) -- (Y);
      \draw[dir] (Z) to [bend right = 0] (X);
      \draw[bidir] (Z) to [bend left=45] (Y);

      \begin{pgfonlayer}{back}
        \draw[fill=betterred!25,draw=none] \convexpath{Z, X}{\outerr mm};
        \node[circle,fill=betterred!65,draw=none,minimum size=2*\innerr mm] at (X) {};
        \node[circle,fill=betterblue!25,draw=none,minimum size=2*\outerr mm] at (Y) {};
        \node[circle,fill=betterblue!65,draw=none,minimum size=2*\innerr mm] at (Y) {};
    \end{pgfonlayer}
  \end{tikzpicture}
  \caption{}
  \label{fig1a}
  \end{subfigure}\hfill
  \begin{subfigure}{0.2\linewidth}\centering
  \begin{tikzpicture}
    \def\outerr{3}
    \def\innerr{2.7}

    \node[vertex] (X) at (0, 0) {$X$};
    \node[vertex] (Z) at (1, 1.5) {$Z$};
    \node[uvertex] (L) at (1, 0.5) {$L$};
    \node[uvertex] (Y) at (2, 0) {$Y$};
    
    \draw[dir] (Z) -- (L);
    \draw[dir] (L) -- (X);
    \draw[dir] (Z) -- (Y);
    \draw[dir] (X) -- (Y);
    \draw[dir] (Z) to [bend right = 0] (X);
    \draw[bidir] (Z) to [bend left=45] (Y);

    \begin{pgfonlayer}{back}
      \draw[fill=betterred!25,draw=none] \convexpath{Z, X}{\outerr mm};
      \node[circle,fill=betterred!65,draw=none,minimum size=2*\innerr mm] at (X) {};
      \node[circle,fill=betterblue!25,draw=none,minimum size=2*\outerr mm] at (Y) {};
      \node[circle,fill=betterblue!65,draw=none,minimum size=2*\innerr mm] at (Y) {};
  \end{pgfonlayer}
\end{tikzpicture}
\caption{}
\label{fig1b}
\end{subfigure}\hfill
\begin{subfigure}{0.2\linewidth}\centering
  \begin{tikzpicture}
    \def\outerr{3}
    \def\innerr{2.7}
    \node[vertex] (X) at (0, 0) {$X$};
    \node[vertex, opacity=0.5] (Xh) at (1, 0.5) {$\hat{X}$};
      \node[vertex] (W) at (0, -1) {$W$};
      \node[vertex] (S) at (1, -1) {$S$};
      \node[uvertex] (Y) at (2, -1) {$Y$};
      
      \draw[dir,opacity=0.5] (Xh) to [bend right=0] (X);
      \draw[dir] (X) -- (W);
      \draw[dir] (W) -- (S);
      \draw[dir] (S) -- (Y);

      \draw[bidir] (X) to [bend left=30] (S);

      \begin{pgfonlayer}{back}
        \draw[fill=betterblue!25,draw=none] \convexpath{S, Y}{\outerr mm};
        \node[circle,fill=betterred!25,draw=none,minimum size=2*\outerr mm] at (X) {};
        \node[circle,fill=betterred!65,draw=none,minimum size=2*\innerr mm] at (X) {};
        \node[circle,fill=betterblue!25,draw=none,minimum size=2*\outerr mm] at (Y) {};
        \node[circle,fill=betterblue!65,draw=none,minimum size=2*\innerr mm] at (Y) {};
    \end{pgfonlayer}
  \end{tikzpicture}
  \caption{}
  \label{fig1c}
  \end{subfigure}\hfill
  \begin{subfigure}{0.2\linewidth}\centering
    \begin{tikzpicture}
      \def\outerr{3}
      \def\innerr{2.7}
      \node[vertex] (X) at (0, 0) {$X$};
      \node[vertex] (Z) at (1, 0.5) {$Z$};
        \node[vertex] (W) at (0, -1) {$W$};
        \node[vertex] (S) at (1, -1) {$S$};
        \node[uvertex] (Y) at (2, -1) {$Y$};
        
        \draw[dir] (X) -- (W);
        \draw[dir] (W) -- (S);
        \draw[dir] (S) -- (Y);
  
        \draw[bidir] (X) to [bend left=30] (S);
        \draw[bidir] (Z) to [bend right=30] (X);
        \draw[bidir] (Z) to [bend left=30] (S);
        \draw[bidir] (Z) to [bend left=10] (W);

        \begin{pgfonlayer}{back}
          \draw[fill=betterred!25,draw=none] \convexpath{Z, X}{\outerr mm};
          \draw[fill=betterblue!25,draw=none] \convexpath{S, Y}{\outerr mm};
          \node[circle,fill=betterred!65,draw=none,minimum size=2*\innerr mm] at (X) {};
          \node[circle,fill=betterblue!25,draw=none,minimum size=2*\outerr mm] at (Y) {};
          \node[circle,fill=betterblue!65,draw=none,minimum size=2*\innerr mm] at (Y) {};
      \end{pgfonlayer}
    \end{tikzpicture}
    \caption{}
    \label{fig1d}
    \end{subfigure}\hfill\null
  \caption{Causal diagrams where $X$ represents an action (shaded red) and $Y$ represents a latent reward (shaded blue). Input covariates of the policy space $\Pi$ are shaded in light red and minimal imitation surrogates relative to action $X$ and reward $Y$ are shaded in light blue. }
  \label{fig1}
\end{figure}
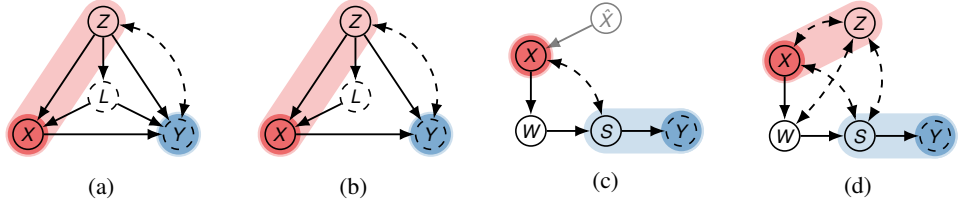

\begin{wrapfigure}{r}{0.3\textwidth}
  \vspace{-0.05in}
    \includegraphics[bb=0 0 750 230, width=0.3\textwidth]{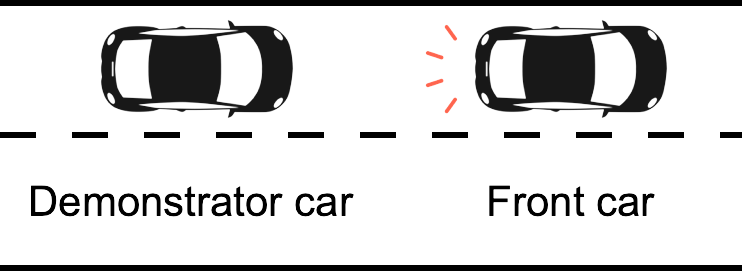}
    \caption{The tail light of the front car is unobserved in highway (aerial) drone data.}
    \label{fig:car}
  \vspace{-0.05in}
\end{wrapfigure}
For concreteness, consider a learning scenario depicted in \Cref{fig:car}, describing trajectories of human-driven cars collected by drones flying over highways \cite{highDdataset,etesami2020causal}. Using such data, we want to learn a policy $\pi(x|z)$ deciding on the acceleration (action) $X$ of the demonstrator car based on the velocity and locations of both the demonstrator and front cars, summarized as covariates $Z$. In reality, the human demonstrator also uses the tail light $L$ of the front car to coordinate his/her actions. The demonstrator's performance is evaluated with a latent reward function $Y$ taking $X, Z, L$ as input. However, only observations of $X, Z$ are collected by the drone, summarized as probabilities $P(x, z)$. \Cref{fig1a} describes the graphical representation of this environment. A na\"ive approach would estimate the conditional distribution $P(x|z)$ and use it as policy $\pi$. A preliminary analysis reveals that this naive ``cloning'' approach leads to sub-optimal performance. Consider an instance where variables $X, Y, Z, L, U \in \{0, 1\}$; their values are decided by functions: $L \gets Z \oplus U$, $X \gets Z \oplus \neg L$, $Y \gets X \oplus Z \oplus L$; $Z, U$ are independent variables drawn uniformly over $\set{0, 1}$; $\oplus$ represents the \textit{exclusive-or} operator. The expected reward $\E[Y|\doo(\pi)]$ induced by $\pi(x|z) = P(x|z)$ is equal to $0.5$, which is quite far from the optimal demonstrator's performance, $\E[Y] = 1$.

This example shows that even when one is able to perfectly mimic an optimal demonstrator, the learned policy can still be suboptimal. In this paper, we try to explicate this phenomenon and, more broadly, understand \textit{imitability} through a causal lens\footnote{Some recent progress in the field of causal imitation has been reported, albeit oblivious to the phenomenon described above and our contributions. Some work considered settings in which the input to the expert policy is fully observed \cite{de2019causal}, while another assumed that the primary outcome is observed (e.g., $Y$ in \Cref{fig1a}) \cite{etesami2020causal}. }. 
Our task is to learn an imitating policy that achieves the expert's performance from demonstration data in a \emph{structural causal model} \cite[Ch.~7]{pearl:2k}, allowing for unobserved confounders (UCs) affecting both action and outcome variables. Specifically, our contributions are summarized as follows. (1) We introduce a complete graphical criterion for determining the feasibility of imitation from demonstration data and qualitative knowledge about the data-generating process represented as a causal graph. (2) We develop a sufficient algorithm for identifying an imitating policy when the given criterion does not hold, by leveraging the quantitative knowledge in the observational distribution. (3) We provide an efficient and practical procedure for finding an imitating policy through explicit parametrization of the causal model, and use it to validate our results on high-dimensional, synthetic datasets. For the sake of space constraints, we provide all proofs in the complete technical report \cite[Appendix A]{appendix}.


\subsection{Preliminaries}
In this section, we introduce the basic notations and definitions used throughout the paper. We use capital letters to denote random variables ($X$) and small letters for their values ($x$). $\D_X$ represents the domain of $X$ and $\PP_X$ the space of probability distributions over $\D_X$. For a set $\*X$, $|\*X|$ denotes its dimension. We consistently use the abbreviation $P(x)$ to represent the probabilities $P(X = x)$. Finally, $I_{\{\ZZ = \zz \}}$ is an indicator function that returns $1$ if $\ZZ = \zz$ holds true; otherwise $0$. 

Calligraphic letters, e.g., $\G$, will be used to represent directed acyclic graphs (DAGs) (e.g., \Cref{fig1}). 
We denote by $\G_{\overline{\*X}}$ the subgraph obtained from $\G$ by removing arrows coming into nodes in $\XX$; $\G_{\underline{\*X}}$ is a subgraph of $\G$ by removing arrows going out of $\*X$. We will use standard family conventions for graphical relationships such as parents, children, descendants, and ancestors. For example, the set of parents of $\*X$ in $\G$ is denoted by $\pa(\*X)_{\G} = \cup_{X \in \*X} \pa(X)_{\G}$. $\ch$, $\de$ and $\an$ are similarly defined, We write $\Pa, \Ch, \De, \An$ if arguments are included as well, e.g. $\De(\XX)_{\G} = \de(\XX)_{\G} \cup \*X$. 
A path from a node $X$ to a node $Y$ in $\G$ is a sequence of edges which does not include a particular node more than once. Two sets of nodes $\*X, \*Y$ are said to be d-separated by a third set $\*Z$ in a DAG $\G$, denoted by $(\*X \ci \*Y | \*Z)_{\G}$, if every edge path from nodes in one set to nodes in another are ``blocked''. The criterion of blockage follows \cite[Def.~1.2.3]{pearl:2k}. 

The basic semantic framework of our analysis rests on \textit{structural causal models} (SCMs) \cite[Ch.~7]{pearl:2k}. An SCM $M$ is a tuple $\tuple{\*U, \*V, \FF, P(\*u)}$ where $\*V$ is a set of endogenous variables and $\*U$ is a set of exogenous variables. $\FF$ is a set of structural functions where $f_V \in \FF$ decides values of an endogenous variable $V \in \*V$ taking as argument a combination of other variables. That is, $V \leftarrow f_{V}(\Pa_V, U_V), \Pa_V \subseteq \*V, U_V \subseteq \*U$. Values of $\bm{U}$ are drawn from an exogenous distribution $P(\*u)$. Each SCM $M$ induces a distribution $P(\*v)$ over endogenous variables $\*V$. An intervention on a subset $\XX \subseteq \VV$, denoted by $\doo(\xx)$, is an operation where values of $\XX$ are set to constants $\xx$, replacing the functions $\{f_{X}: \forall X \in \XX\}$ that would normally determine their values. For an SCM $M$, let $M_{\xx}$ be a submodel of $M$ induced by intervention $\doo(\xx)$. For a set $\*S \subseteq \*V$, the interventional distribution $P(\sss|\doo(\xx))$ induced by $\doo(\*x)$ is defined as the distribution over $\*S$ in the submodel $M_{\xx}$, i.e., $P(\*s|\doo(\xx); M) \triangleq P(\*s;M_{\xx})$. We leave $M$ implicit when it is obvious from the context. For a detailed survey on SCMs, we refer readers to \cite[Ch.~7]{pearl:2k}.

\section{Imitation Learning in Structural Causal Models}
In this section, we formalize and study the imitation learning problem in causal language. We first define a special type of SCM that explicitly allows one to model the unobserved nature of some endogenous variables,  which is called the partially observable structural causal model (POSCM).\footnote{This definition will facilitate the more explicitly articulation of which endogenous variables are available to the demonstrator and corresponding policy at each point in time. } 
\begin{definition}[Partially Observable SCM]
A POSCM is a tuple $\tuple{M, \*O, \*L}$, where $M$ is a SCM $\tuple{\*U, \*V, \1F, P(\*u)}$ and $\tuple{\*O, \*L}$ is a pair of subsets forming a partition over $\*V$ (i.e., $\*V = \*O \cup \*L$ and $\*O \cap \*L = \emptyset$); $\*O$ and $\*L$ are called observed and latent endogenous variables, respectively.
\end{definition}

Each POSCM $M$ induces a probability distribution over $V$, of which one can measure the observed variables $\*O$. $P(\*o)$ is usually called the \textit{observational} distribution. $M$ is associated with a \textit{causal diagram} $\G$ (e.g., see \Cref{fig1}) where solid nodes represent observed variables $\*O$, dashed nodes represent latent variables $\*L$, and arrows represent the arguments $\Pa_V$ of each functional relationship $f_{V}$. Exogenous variables $\*U$ are not explicitly shown; a bi-directed arrow between nodes $V_i$ and $V_j$ indicates the presence of an unobserved confounder (UC) affecting both $V_i$ and $V_j$, i.e., $U_{V_i} \cap U_{V_j} \neq \emptyset$. 

Consider a POSCM $\tuple{M, \*O, \*L}$ with $M = \tuple{\*U, \*V, \1F, P(\*u)}$. Our goal is to learn an efficient policy to decide the value of an action variable $X \in \*O$. The performance of the policy is evaluated using the expected value of a reward variable $Y$. Throughout this paper, we assume that reward $Y$ is latent and $X$ affects $Y$ (i.e., $Y \in \*L \cap \De(X)_{\G}$). A \emph{policy} $\pi$ is a function mapping from values of covariates $\Pa^* \subseteq \*O \setminus \De(X)_{\G_{\overline{X}}}$\footnote{$\G_{\overline{X}}$ is a causal diagram associated with the submodel $M_x$ induced by intervention $do(x)$.} to a probability distribution over $X$, which we denote by $\pi(x|\pa^*)$. An intervention following a policy $\pi$, denoted by $\doo(\pi)$, is an operation that draws values of X independently following $\pi$, regardless of its original (natural) function $f_X$. Let $M_{\pi}$ denote the manipulated SCM of $M$ induced by $\doo(\pi)$. Similar to atomic settings, the interventional distribution $P(\*v|\doo(\pi))$ is defined as the distribution over $\*V$ in the manipulated model $M_{\pi}$, given by,
\begin{align}
  P(\*v|\doo(\pi)) = \sum_{\*u} P(\*u) \prod_{ V \in \*V \setminus \set{X}} P(v|\pa_V, u_V) \pi(x |\pa^*). \label{eq:do}
\end{align}
The expected reward of a policy $\pi$ is thus given by the causal effect $\E[Y|\doo(\pi)]$. The collection of all possible policies $\pi$ defines a \textit{policy space}, denoted by $\Pi = \set{\pi: \D_{\Pa^*} \mapsto \PP_{X} }$ (if $\Pa^* = \emptyset$, $\Pi = \set{\pi: \PP_X}$). For convenience, we define function $\Pa(\Pi) = \Pa^*$. A policy space $\Pi'$ is a subspace of $\Pi$ if $\Pa(\Pi') \subseteq \Pa(\Pi)$.
We will consistently highlight action $X$ in dark red, reward $Y$ in dark blue and covariates $\Pa(\Pi)$ in light red.  For instance, in \Cref{fig1a}, the policy space over action $X$ is given by $\Pi = \set{\pi: \D_{Z} \mapsto \PP_X}$; $Y$ represents the reward; $\Pi' = \set{\pi: \PP_X}$ is a subspace of $\Pi$.

Our goal is to learn an efficient policy $\pi \in \Pi$ that achieves satisfactory performance, e.g., larger than a certain threshold $\E[Y|\doo(\pi)] \geq \tau$, without knowledge of underlying system dynamics, i.e., the actual, true POSCM $M$. A possible approach is to identify the expected reward $\E[Y|\doo(\pi)]$ for each policy $\pi \in \Pi$ from the combinations of the observed data $P(\*o)$ and the causal diagram $\G$. Optimization procedures are applicable to find a satisfactory policy $\pi$. Let $\2M_{\langle \G \rangle}$ denote a hypothesis class of POSCMs that are compatible with a causal diagram $\G$. We define the non-parametric notion of identifiability in the context of POSCMs and conditional policies, adapted from \cite[Def.~3.2.4]{pearl:2k}.
\begin{definition}[Identifiability]\label{def:id}
  Given a causal diagram $\G$ and a policy space $\Pi$, let $\*Y$ be an arbitrary subset of $\*V$. $P(\*y|\doo(\pi))$ is said to be identifiable w.r.t. $\tuple{\G, \Pi}$ if $P\left(\*y|\doo(\pi); M \right)$ is uniquely computable from $P(\*o; M)$ and $\pi$ for any POSCM $M \in \2M_{\langle \G \rangle}$ and any $\pi \in \Pi$.
\end{definition}
In imitation learning settings, however, reward $Y$ is often not specified and remains latent, which precludes approaches that attempt to identify $\E[Y|\doo(\pi)]$:
\begin{restatable}{corollary}{corolnonid}\label{corol:nonid}
Given a causal diagram $\G$ and a policy space $\Pi$, let $\*Y$ be an arbitrary subset of $\*V$. If not all variables in $\*Y$ are observed (i.e., $\*Y \cap \*L \neq \emptyset$), $P(\*y|\doo(\pi))$ is not identifiable.
\end{restatable}
In other words, \Cref{corol:nonid} shows that when the reward $Y$ is latent, it is infeasible to uniquely determine values of $\E[Y|\doo(\pi)]$ from $P(\*o)$. A similar observation has been noted in \cite[Prop.~1]{lee2020causal}. This suggests that we need to explore learning through other modalities.

\subsection{Causal Imitation Learning}
To circumvent issues of non-identifiability, a common solution is to assume that the observed trajectories are generated by an ``expert'' demonstrator with satisfactory performance $\E[Y]$, e.g., no less than a certain threshold ($\E[Y] \geq \tau$). If we could find a policy $\pi$ that perfectly ``imitates" the expert with respect to reward $Y$, $\E[Y|\doo(\pi)] = \E[Y]$, the performance of the learner is also guaranteed to be satisfactory. Formally,
\begin{definition}[Imitability] \label{def:imitability}
  Given a causal diagram $\G$ and a policy space $\Pi$, let $\*Y$ be an arbitrary subset of $\*V$. $P(\*y)$ is said to be imitable w.r.t. $\tuple{\G, \Pi}$ if there exists a policy $\pi \in \Pi$ uniquely computable from $P(\*o)$ such that $P(\*y|\doo(\pi); M) = P(\*y; M)$ for any POSCM $M \in \2M_{\langle \G \rangle}$.
\end{definition}
Our task is to determine the imitability of the expert performance. More specifically, we want to learn an \emph{imitating policy} $\pi \in \Pi$ from $P(\*o)$ such that $P(y|\doo(\pi)) = P(y)$\footnote{The imitation is trivial if $Y \not \in \De(X)_{\G}$: by Rule 3 of \cite[Thm.~3.4.1]{pearl:2k} (or \cite[Thm.~1]{correa2020calculus}), $P(y|\doo(\pi)) = P(y)$ for any policy $\pi$. This paper aims to find a \emph{specific} $\pi$ satisfying $P(y|\doo(\pi)) = P(y)$ even when $Y \in \De(X)_{\G}$.}, in any POSCM $M$ associated with the causal diagram $\G$. Consider \Cref{fig2a} as an example. $P(y)$ is imitable with policy $\pi(x) = P(x)$ since by \Cref{eq:do} and marginalization, $P(y|\doo(\pi)) = \sum_{x,w} P(y|w)P(w|x)\pi(x) = \sum_{x,w}P(y|w)P(w|x)P(x) = P(y)$. In practice, unfortunately, the expert's performance cannot always be imitated. To understand this setting, we first write, more explicitly, the conditions under which this is not the case:
\begin{restatable}{lemma}{lemnonimitable}\label{lem:non-imitable}
  Given a causal diagram $\G$ and a policy space $\Pi$, let $\*Y$ be an arbitrary subset of $\*V$. $P(\*y)$ is not imitable w.r.t. $\tuple{\G, \Pi}$ if there exists two POSCMs $M_1, M_2 \in \2M_{\langle \G \rangle}$ satisfying $P(\*o; M_1) = P(\*o; M_2)$ while there exists no policy $\pi \in \Pi$ such that for $i = 1, 2$, $P(\*y|\doo(\pi); M_i) = P(\*y;M_i)$.
\end{restatable}
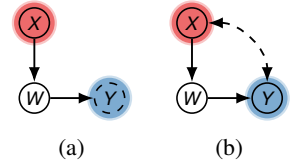
\begin{wrapfigure}{r}{0.3\textwidth}
  \vspace{-0.1in}
  \hfill
  \begin{subfigure}{0.15\textwidth}\centering
      \begin{tikzpicture}
        \def\outerr{3}
        \def\innerr{2.7}
        \node[vertex] (X) at (0, 0) {$X$};
        \node[vertex] (W) at (0, -1) {$W$};
        \node[uvertex] (Y) at (1, -1) {$Y$};
        
        \draw[dir] (X) -- (W);
        \draw[dir] (W) -- (Y);

        \begin{pgfonlayer}{back}
          \node[circle,fill=betterred!25,draw=none,minimum size=2*\outerr mm] at (X) {};
          \node[circle,fill=betterred!65,draw=none,minimum size=2*\innerr mm] at (X) {};
          \node[circle,fill=betterblue!25,draw=none,minimum size=2*\outerr mm] at (Y) {};
          \node[circle,fill=betterblue!65,draw=none,minimum size=2*\innerr mm] at (Y) {};
      \end{pgfonlayer}
      \end{tikzpicture}
    \caption{}
    \label{fig2a}
    \end{subfigure}\hfill
    \begin{subfigure}{0.15\textwidth}\centering
      \begin{tikzpicture}
        \def\outerr{3}
        \def\innerr{2.7}
        \node[vertex] (X) at (0, 0) {$X$};
        \node[vertex] (W) at (0, -1) {$W$};
        \node[vertex] (Y) at (1, -1) {$Y$};
        
        \draw[dir] (X) -- (W);
        \draw[dir] (W) -- (Y);
        \draw[bidir] (X) to [bend left = 45] (Y);
        \begin{pgfonlayer}{back}
          \node[circle,fill=betterred!25,draw=none,minimum size=2*\outerr mm] at (X) {};
          \node[circle,fill=betterred!65,draw=none,minimum size=2*\innerr mm] at (X) {};
          \node[circle,fill=betterblue!25,draw=none,minimum size=2*\outerr mm] at (Y) {};
          \node[circle,fill=betterblue!65,draw=none,minimum size=2*\innerr mm] at (Y) {};
      \end{pgfonlayer}
      \end{tikzpicture}
    \caption{}
    \label{fig2b}
    \end{subfigure}\hfill\null
    \caption{Imitability v. Identifiability.}
    \label{fig2}
    \vspace{-0.2in}
\end{wrapfigure}
It follows as a corollary that $P(\*y)$ is not imitable if there exists a POSCM $M$ compatible with $\G$ such that no policy $\pi \in \Pi$ could ensure $P(\*y|\doo(\pi);M) = P(\*y; M)$. For instance, consider the causal diagram $\G$ and policy space $\Pi$ in \Cref{fig2b}. Here, the expert's reward $P(y)$ is not imitable: consider a POSCM with functions $X \gets U, W \gets X, Y \gets U \oplus \neg W$; values $U$ are drawn uniformly over $\{0, 1\}$. In this model, $P(Y=1|\doo(\pi)) = 0.5$ for any policy $\pi$, which is far from the optimal expert reward, $P(Y = 1)= 1$.

An interesting observation from the above example of \Cref{fig2b} is that the effect $P(y|\doo(\pi))$ is identifiable, following the front-door criterion in \cite[Thm.~3.3.4]{pearl:2k}, but no policy imitates the corresponding $P(y)$. However, in some settings, the expert's reward $P(y)$ is imitable but the imitator's reward $P(y|\doo(\pi))$ cannot be uniquely determined. To witness, consider again the example in \Cref{fig2a}. The imitability of $P(y)$ has been previously shown; while $P(y|\doo(\pi))$ is not identifiable due to latent reward $Y$ (\Cref{corol:nonid}).

In general, the problem of imitability is orthogonal to identifiability, and, therefore, requires separate consideration. Since imitability does not always hold, we introduce a useful graphical criterion for determining whether imitating an expert's performance is feasible, and if so, how.
\begin{restatable}[Imitation by Direct Parents]{theorem}{thmdp}\label{thm:dp}
  Given a causal diagram $\G$ and a policy space $\Pi$, $P(y)$ is imitable w.t.r. $\tuple{\G, \Pi}$ if $\pa(X)_{\G} \subseteq \Pa(\Pi)$ and there is no bi-directed arrow pointing to $X$ in $\G$. Moreoever, the imitating policy $\pi \in \Pi$ is given by $\pi\left(x | \pa(\Pi)\right) = P\left(x | \pa(X)_{\G} \right)$.
\end{restatable}
In words, \Cref{thm:dp} says that if the expert and learner share the same policy space, then the policy is always imitable. In fact, this result can be seen as a causal justification for when the method of ``behavior cloning'', widely used in practice, is valid, leading to proper imitation. When the original behavior policy $f_X$ is contained in the policy space $\Pi$, the leaner could imitate the expert's reward $P(y)$ by learning a policy $\pi \in \Pi$ that matches the distribution $P(x|\pa(\Pi))$ \cite{widrow1964pattern,pomerleau1989alvinn}. Next, we consider the more challenging setting when policy spaces of the expert and learner disagree (i.e., the learner and expert have different views of the world, $f_X \not \in \Pi$). We will leverage a graphical condition adopted from the celebrated \textit{backdoor criterion} \cite[Def.~3.3.1]{pearl:2k}.
\begin{definition}[$\pi$-Backdoor]\label{def:backdoor}
Given a causal diagram $\G$ and a policy space $\Pi$, a set $\*Z$ is said to satisfy the \emph{$\pi$-backdoor} criterion w.r.t. $\tuple{\G. \Pi}$ if and only if $\*Z \subseteq \Pa(\Pi)$ and $(Y \ci X | \*Z)_{\G_{\underline{X}}}$, which is called the \emph{$\pi$-backdoor admissible set} w.r.t. $\tuple{\G, \Pi}$.
\end{definition}
For concreteness, consider again the highway driving example in \Cref{fig1a}. There exists no $\pi$-backdoor admissible set due to the path $X \leftarrow L \rightarrow Y$. Now consider a modified graph in \Cref{fig1b} where edge $L \rightarrow Y$ is removed. $\{Z\}$ is $\pi$-backdoor admissible since $Z \in \Pa(\Pi)$ and $\left (Y \ci X | Z \right)_{\G_{\underline{X}}}$. Leveraging the imitation backdoor condition, our next theorem provides a full characterization for when imitating expert's performance is achievable, despite the fact that the reward $Y$ is latent.
\begin{restatable}[Imitation by $\pi$-Backdoor]{theorem}{thmbackdoor}\label{thm:backdoor}
  Given a causal diagram $\G$ and a policy space $\Pi$, $P(y)$ is imitable w.r.t. $\tuple{\G, \Pi}$ if and only if there exists an $\pi$-backdoor admissible set $\*Z$ w.r.t. $\tuple{\G, \Pi}$. Moreover, the imitating policy $\pi \in \Pi$ is given by $\pi\left(x | \pa(\Pi) \right) = P\left(x | \*z \right)$.
\end{restatable}
That is, one can learn an imitating policy from a policy space $\Pi' = \{\pi: \D_{\*Z} \mapsto \PP_{X}\}$ that mimics the conditional probabilities $P(x|\*z)$ if and only if $\*Z$ is $\pi$-backdoor admissible. If that is the case, such a policy can be learned from data through standard density estimation methods. For instance, \Cref{thm:backdoor} ascertains that $P(y)$ in \Cref{fig1a} is indeed non-imitable. On the other hand, $P(y)$ in \Cref{fig1b} is imitable, guaranteed by the $\pi$-backdoor admissible set $\set{Z}$; the imitating policy is given by $\pi(x|z) = P(x|z)$.

\section{Causal Imitation Learning with Data Dependency}
One may surmise that the imitation boundary established by \Cref{thm:backdoor} suggests that when there exists no $\pi$-backdoor admissible set, it is infeasible to imitate the expert performance from observed trajectories of demonstrations. In this section, we will circumvent this issue by exploiting actual parameters of the observational distribution $P(\*o)$. In particular, we denote by $\2M_{\langle \G, P \rangle}$ a subfamily of candidate models in $\2M_{\tuple{\G}}$ that induce both the causal diagram $\G$ and the observational distribution $P(\*o)$, i.e., $\2M_{\langle \G, P \rangle} = \left \{\forall M \in \2M_{\langle \G \rangle}: P(\*o;M) = P(\*o)\right \}$. We introduce a refined notion of imitability that will explore the quantitative knowledge of observations $P(\*o)$ (to be exemplified). Formally,
\begin{definition}[Practical Imitability]\label{def:p-imtability}
  Given a causal diagram $\G$, a policy space $\Pi$, and an observational distribution $P(\*o)$, let $\*Y$ be an arbitrary subset of $\*V$. $P(\*y)$ is said to be \emph{practically imitable} (for short, p-imitable) w.r.t. $\tuple{\G, \Pi, P(\*o)}$ if there exists a policy $\pi \in \Pi$ uniquely computable from $P(\*o)$ such that $P(\*y|\doo(\pi); M) = P(\*y;M)$ for any POSCM $M \in \2M_{\langle \G, P \rangle}$.
\end{definition}
The following corollary can be derived based on the definition of practical imitability.
\begin{restatable}{corollary}{corolimtability}\label{corol:p-imtability}
  Given a causal diagram $\G$, a policy space $\Pi$ and an observational distribution $P(\*o)$, let a subset $\*Y \subseteq \*V$. If $P(\*y)$ is imitable w.r.t. $\tuple{\G, \Pi}$, $P(\*y)$ is p-imitable w.r.t. $\tuple{\G, \Pi, P(\*o)}$.
\end{restatable}
Compared to \Cref{def:imitability}, the practical imitability of \Cref{def:p-imtability} aims to find an imitating policy for a subset of candidate POSCMs $\2M_{\langle \G, P \rangle}$ restricted to match a specific observational distribution $P(\*o)$. \Cref{def:imitability}, on the other hand, requires only the causal diagram $\G$. In other words, for an expert's performance $P(y)$ that is non-imitable w.r.t. $\tuple{\G, \Pi}$, it could still be p-imitable after analyzing actual probabilities of the observational distribution $P(\*o)$.  

For concreteness, consider again $P(y)$ in \Cref{fig2b} which is not imitable due to the bi-directed arrow $X \leftrightarrow Y$. However, new imitation opportunities arise when actual parameters of the observational distribution $P(x, w, y)$ are provided. Suppose the underlying POSCM is given by: $X \gets U_X \oplus U_Y$, $W \gets X \oplus U_W$, $Y \gets W \oplus U_Y$ where $U_X, U_Y, U_W$ are independent binary variables drawn from $P(U_X = 1) = P(U_Y = 1) = P(U_W = 0) = 0.9$. Here, the causal effect $P(y|\doo(x))$ is identifiable from $P(x, w, y)$ following the front-door formula $P(y|\doo(x)) = \sum_{w} P(w|x)\sum_{x'} P(y|w, x')P(x')$ \cite[Thm.~3.3.4]{pearl:2k}. We thus have $P(Y = 1 | \doo(X=0)) = 0.82$ which coincides with $P(Y = 1) = 0.82$, i.e., $P(y)$ is p-imitable with atomic intervention $\doo(X = 0)$. In the most practical settings, the expert reward $P(y)$ rarely equates to $P(y|\doo(x))$; stochastic policies $\pi(x)$ are then applicable to imitate $P(y)$ by re-weighting $P(y|\doo(x))$ induced by the corresponding atomic interventions \footnote{Consider a variation of the model where $P(U_W = 1) = 0.7$. $P(y)$ is p-imitable with $\pi(X = 0) = 0.75$.}. 
To tackle p-imitability in a general way, we proceed by defining a set of observed variables that serve as a surrogate of the unobserved $Y$ with respect to interventions on $X$. Formally, 
\begin{definition}[Imitation Surrogate]\label{def:surrogate}
  Given a causal diagram $\G$, a policy space $\Pi$, let $\*S$ be an arbitrary subset of $\*O$. $\*S$ is an \emph{imitation surrogate} (for short, surrogate) w.r.t. $\tuple{\G, \Pi}$ if $(Y \ci \hat{X} | \*S)_{\G \cup \Pi}$ where $\G \cup \Pi$ is a supergraph of $\G$ by adding arrows from $\Pa(\Pi)$ to $X$; $\hat{X}$ is a new parent to $X$.
\end{definition}
An surrogate $\*S$ is said to be minimal if there exists no subset $\*S' \subset \*S$ such that $\*S'$ is also a surrogate w.r.t. $\tuple{\G, \Pi}$. Consider as an example \Cref{fig1c} where the supergraph $\G \cup \Pi$ coincides with the causal diagram $\G$. By \Cref{def:surrogate}, both $\{W, S\}$ and $\{S\}$ are valid surrogate relative to $\tuple{X, Y}$ with $\{S\}$ being the minimal one. By conditioning on $S$, the decomposition of \Cref{eq:do} implies $P(y | \doo(\pi)) = \sum_{s, w, u} P(y|s) P(s| w, u)P(w|x)\pi(x) P(u) = \sum_{s} P(y|s) P(s|\doo(\pi))$. That is, the surrogate $S$ mediates all influence of interventions on action $X$ to reward $Y$. It is thus sufficient to find an imitating policy $\pi$ such that $P(s|\doo(\pi)) = P(s)$ for any POSCM $M$ associated with \Cref{fig1c}. The resultant policy is guaranteed to imitate the expert's reward $P(y)$. 

When a surrogate $\*S$ is found and $P(\*s|\doo(\pi))$ is identifiable, one could compute $P(\*s|\doo(\pi))$ for each policy $\pi$ and check if it matches $P(\*s)$. In many settings, however, $P(\*s|\doo(\pi))$ is not identifiable w.r.t. $\tuple{\G, \Pi}$. For example, in \Cref{fig1d}, $S$ is a surrogate w.r.t. $\tuple{\G, \Pi}$, but $P(s|\doo(\pi))$ is not identifiable due to collider $Z$ ($\pi$ uses non-descendants as input by default). Fortunately, identifying $P(\*s|\doo(\pi))$ may still be feasible in some subspaces of $\Pi$:
\begin{definition}[Identifiable Subspace]\label{def:idsp}
  Given a causal diagram $\G$, a policy space $\Pi$, and a subset $\*S \subseteq \*O$, let $\Pi'$ be a policy subspace of $\Pi$. $\Pi'$ is said to be an \textit{identifiable subspace} (for short, id-subspace) w.r.t. $\tuple{\G, \Pi, \*S}$ if $P(\*s|\doo(\pi))$ is identifiable w.r.t. $\tuple{\G, \Pi'}$.
\end{definition}
Consider a policy subspace $\Pi' = \set{\pi: \PP_X}$ described in \Cref{fig1d} (i.e. $\pi$ that does not exploit information from covariates $Z$). $P(s|\doo(\pi))$ is identifiable w.r.t. $\tuple{\G, \Pi'}$ following the front-door adjustment on $W$ \cite[Thm.~3.3.4]{pearl:2k}. We could then evaluate interventional probabilities $P(s | \doo(\pi))$ for each policy $\pi \in \Pi'$ from the observational distribution $P(x, w, s, z)$; the imitating policy is obtainable by solving the equation $P(s|\doo(\pi)) = P(s)$. In other words, $\set{S}$ and $\Pi'$ forms an instrument that allows one to solve the imitation learning problem in \Cref{fig1d}.
\begin{definition}[Imitation Instrument]\label{def:i-cond}
  Given a causal diagram $\G$ and a policy space $\Pi$, let $\*S$ be a subset of $\*O$ and $\Pi'$ be a subspace of $\Pi$. $\tuple{\*S, \Pi'}$ is said to be an \emph{imitation instrument} (for short, instrument) if $\*S$ is a surrogate w.r.t. $\tuple{\G, \Pi'}$ and $\Pi'$ is an id-subspace w.r.t. $\tuple{\G, \Pi, \*S}$. 
\end{definition}
\begin{restatable}{lemma}{lemsurrogate}\label{lem:surrogate}
  Given a causal diagram $\G$, a policy space $\Pi$, and an observational distribution $P(\*o)$, let $\tuple{\*S, \Pi'}$ be an instrument w.r.t. $\tuple{\G, \Pi}$. If $P(\*s)$ is p-imitable w.r.t. $\tuple{\G, \Pi', P(\*o)}$, then $P(y)$ is p-imitable w.r.t. $\tuple{\G, \Pi, P(\*o)}$. Moreover, an imitating policy $\pi$ for $P(\*s)$ w.r.t. $\tuple{\G, \Pi', P(\*o)}$ is also imitating policy for $P(y)$ w.r.t. $\tuple{\G, \Pi, P(\*o)}$.
\end{restatable}
In words, \Cref{lem:surrogate} shows that when an imitation instrument $\tuple{\*S, \Pi'}$ is present, we could reduce the original imitation learning on a latent reward $Y$ to a p-imitability problem over observed surrogate variables $\*S$ using policies in an identifiable subspace $\Pi'$. The imitating policy $\pi$ is obtainable by solving the equation $P(\*s| \doo(\pi)) = P(\*s)$.

\subsection{Confounding Robust Imitation}\label{sec3.1}
Our task in this section is to introduce a general algorithm that finds instruments, and learns a p-imitating policy given $\tuple{\G, \Pi, P(\*o)}$. A na\"ive approach is to enumerate all pairs of subset $\*S$ and subspace $\Pi'$ and check whether they form an instrument; if so, we can compute an imitating policy for $P(\*s)$ w.r.t. $\tuple{\G, \Pi', P(\*o)}$. However, the challenge is that the number of all possible subspaces $\Pi'$ (or subsets $\*S$) can be exponentially large. Fortunately, we can greatly restrict this search space. Let $\G \cup \set{Y}$ denote a causal diagram obtained from $\G$ by making reward $Y$ observed. The following proposition suggests that it suffices to consider only identifiable subspaces w.r.t. $\tuple{\G \cup \set{Y}, \Pi, Y}$. 
\begin{restatable}{lemma}{lemidY}\label{lem:idY}
  Given a causal diagram $\G$, a policy space $\Pi$, let a subspace $\Pi' \subseteq \Pi$. If there exists $\*S \subseteq \*O$ such that $\tuple{\*S, \Pi'}$ is an instrument w.r.t. $\tuple{\G, \Pi}$, $\Pi'$ is an id-subspace w.r.t. $\tuple{\G \cup \set{Y}, \Pi, Y}$.
\end{restatable}
\begin{wrapfigure}{r}{0.55\textwidth}
  \vspace{-0.2in}
  \begin{algorithm}[H]
    \caption{\textsc{Imitate}}
    \label{imitate}
  \begin{algorithmic}[1]
   \STATE {\bfseries Input:} $\G, \Pi, P(\*o)$.
   \WHILE{$\textsc{ListIdSpace}(\G \cup \set{Y}, \Pi, Y)$ outputs a policy subspace $\Pi'$}
   \WHILE{$\textsc{ListMinSep}(\G \cup \Pi', \hat{X}, Y, \set{}, \*O)$ outputs a surrogate set $\*S$}
   \IF{$\textsc{Identify}(\G, \Pi', \*S) = \textsc{Yes}$}
   \STATE Solve for a policy $\pi \in \Pi'$ such that
   \[
    P(\*s|\doo(\pi);M) = P(\*s)
   \]
   \INDENT
   for any POSCM $M \in \2M_{\langle \G, P \rangle}$.
   \ENDINDENT
   \STATE Return $\pi$ if it exists; continue otherwise.
   \ENDIF
   \ENDWHILE
   \ENDWHILE \text{ Return} \textsc{Fail}.
  \end{algorithmic}
  \end{algorithm}
  \vspace{-0.2in}
\end{wrapfigure}
Our algorithm \textsc{Imitate} is described in \Cref{imitate}. We assume access to an \textsc{Identify} oracle \cite{tian:pea03-r290,shpitser:pea06a,correa2020calculus} that takes as input a causal diagram $\G$, a policy space $\Pi$ and a set of observed variables $\*S$. If $P(\*s|\doo(\pi))$ is identifiable w.r.t. $\tuple{\G, \Pi}$, \textsc{Identify} returns ``\textsc{Yes}''; otherwise, it returns ``\textsc{No}''. For details about the \textsc{Identify} oracle, we refer readers to \cite[Appendix B]{appendix}. More specifically, \textsc{Imitate} takes as input a causal diagram $\G$, a policy space $\Pi$ and an observational distribution $P(\*o)$. At Step 2, \textsc{Imitate} applies a subroutine \textsc{ListIdSpace} to list identifiable subspaces $\Pi'$ w.r.t. $\tuple{\G \cup \set{Y}, \Pi, Y}$, following the observation made in \Cref{lem:idY}. The implementation details of \textsc{ListIdSpace} are provided in \cite[Appendix C]{appendix}. When an identifiable subspace $\Pi'$ is found, \textsc{Imitate} tries to obtain a surrogate $\*S$ w.r.t the diagram $\G$ and subspace $\Pi'$. While there could exist multiple such surrogates, the following proposition shows that it is sufficient to consider only minimal ones.
\begin{restatable}{lemma}{lemminimal}\label{lem:minimal}
  Given a causal diagram $\G$, a policy space $\Pi$, an observational distribution $P(\*o)$ and a subset $\*S \subseteq \*O$. $P(\*s)$ is p-imitable only if for any $\*S' \subseteq \*S$, $P(\*s')$ is p-imitable w.r.t. $\tuple{\G, \Pi, P(\*o)}$.
\end{restatable}
We apply a subroutine \textsc{ListMinSep} in \cite{van2014constructing} to enumerate minimal surrogates in $\*O$ that d-separate $\hat{X}$ and $Y$ in the supergraph $\G \cup \Pi'$. When a minimal surrogate $\*S$ is found, \textsc{Imitate} uses the \textsc{Identify} oracle to validate if $P(\*s|\doo(\pi))$ is identifiable w.r.t. $\tuple{\G, \Pi'}$, i.e., $\tuple{\*S, \Pi'}$ form an instrument. Consider \Cref{fig1d} as an example. While $P(y|\doo(\pi))$ is not identifiable for every policy in $\Pi$ had $Y$ been observed, $\Pi$ contains an id-subspace $\{\pi: \PP_{X}\}$ w.r.t. $\tuple{\G \cup \set{Y}, \Pi, Y}$, which is associated with a minimal surrogate $\{S\}$. Applying \textsc{Identify} confirms that $\tuple{\set{S}, \{\pi: \PP_{X}\}}$ is an instrument.

At Step 5, \textsc{Imitate} solves for a policy $\pi$ in the subspace $\Pi'$ that imitates $P(\*s)$ for all instances in the hypothesis class $\2M_{\langle \G, P \rangle}$. If such a policy exists, \textsc{Imitate} returns $\pi$; otherwise, the algorithm continues. Since $\tuple{\*S, \Pi'}$ is an instrument, \Cref{lem:surrogate} implies that the learned policy $\pi$, if it exists, is ensured to imitate the expert reward $P(y)$ for any POSCM $M \in \2M_{\langle \G, P \rangle}$.
\begin{restatable}{theorem}{thmpimitation}\label{thm:p-imitation}
  Given a causal diagram $\G$, a policy space $\Pi$, and an observational distribution $P(\*o)$, if \textsc{Imitate} returns a policy $\pi \in \Pi$, $P(y)$ is p-imitable w.r.t. $\tuple{\G, \Pi, P(\*o)}$. Moreover, $\pi$ is an imitating policy for $P(y)$ w.r.t. $\tuple{\G, \Pi, P(\*o)}$.
\end{restatable}

\subsection{Optimizing Imitating Policies}
We now introduce optimization procedures to solve for an imitating policy at Step 5 of \textsc{Imitate} algorithm. Since the pair $\tuple{\*S, \Pi'}$ forms a valid instrument (ensured by Step 4), the interventional distribution $P(\*s|\doo(\pi); M)$ remains invariant among all models in $\2M_{\langle \G \rangle}$, i.e., $P(\*s|\doo(\pi))$ is identifiable w.r.t. $\tuple{\G, \Pi}$. We could thus express $P(\*s|\doo(\pi); M)$ for any $M \in  \2M_{\tuple{\G, P}}$ as a function of the observational distribution $P(\*o)$; for simplicity, we write $P(\*s|\doo(\pi)) = P(\*s|\doo(\pi); M)$. The imitating policy $\pi$ is obtainable by solving the equation $P(\*s|\doo(\pi)) = P(\*s)$. We could derive a closed-form formula for $P(\*s|\doo(\pi))$ following standard causal identification algorithms in \cite{tian:pea03-r290,shpitser:pea06a,correa2020calculus}. As an example, consider again the setting of \Cref{fig1c} with binary $X, W, S, Z$; parameters of $P(x, w, s,z)$ could be summarized using an $8$-entry probability table. The imitating policy $\pi(x)$ is thus a solution of a series of linear equations $\sum_{x}\pi(x) P(s|\doo(x)) = P(s)$ and $\sum_x \pi(x) = 1$, given by:
\begin{align*}
  &\pi(x_0) = \frac{P(s_1) - P(s_1|\doo(x_0))}{P(s_1|\doo(x_1)) - P(s_1|\doo(x_0))}, &&&\pi(x_1) = \frac{P(s_1 | \doo(x_1)) - P(s_1)}{P(s_1|\doo(x_1)) - P(s_1|\doo(x_0))}.
\end{align*}
Among quantities in the above equation, $x_i, s_j$ represent assignments $X = i, S = j$ for $i, j \in \{0, 1\}$. The interventional distribution $P(s|\doo(x))$ could be identified from $P(x, w, s, z)$ using the front-door adjustment formula $P(s|\doo(x)) = \sum_{w} P(w|x) \sum_{x'} P(s|x', w)P(x')$ \cite[Thm.~3.3.4]{pearl:2k}.

However, evaluating interventional probabilities $P(\*s|\doo(\pi))$ from the observational distribution $P(\*o)$ could be computational challenging if some variables in $\*O$ are high-dimensional (e.g., $W$ in \Cref{fig1d}). Properties of the imitation instrument $\tuple{\*S, \Pi'}$ suggest a practical approach to address this issue. Since $P(\*s|\doo(\pi))$ is identifiable w.r.t. $\tuple{\G, \Pi'}$, by \Cref{def:id}, it remains invariant over the models in the hypothesis class $\2M_{\langle \G \rangle}$. This means that we could compute interventional probabilities $P(\*s | \doo(\pi); \tilde{M})$ in an arbitrary model $\tilde{M} \in \2M_{\langle \G, P \rangle}$; such an evaluation will always coincide with the actual, true causal effect $P(\*s|\doo(\pi))$ in the underlying model. This observation allows one to obtain an imitating policy through the direct parametrization of POSCMs \cite{louizosCausalEffectInference2017}. Let $\2N_{\tuple{\G}}$ be a parametrized subfamily of POSCMs in $\2M_{\langle \G \rangle}$. We could obtain an POSCM $\tilde{M} \in \2N_{\tuple{\G}}$ such that its observational distribution $P(\*o;\tilde{M}) = P(\*o)$; an imitating policy $\pi$ is then computed in the parametrized model $\tilde{M}$. \Cref{corol:parametric} shows that such a policy $\pi$ is an imitating policy for the expert's reward $P(y)$.
\begin{restatable}{corollary}{corolparametric}\label{corol:parametric}
  Given a causal diagram $\G$, a policy space $\Pi$, and an observational distribution $P(\*o)$, let $\tuple{\*S, \Pi'}$ be an instrument w.r.t. $\tuple{\G, \Pi}$. If there exists a POSCM $M \in \2M_{\tuple{\G, P}}$ and a policy $\pi \in \Pi'$ such that $P(\*s|\doo(\pi); M) = P(\*s)$, then $P(y)$ is p-imitable w.r.t. $\tuple{\G, \Pi, P(\*o)}$. Moreover, $\pi$ is an imitating policy for $P(y)$ w.r.t. $\tuple{\G, \Pi, P(\*o)}$. 
\end{restatable}
In practical experiments, we consider a parametrized family of POSCMs $\2N_{\tuple{\G}}$ where functions associated with each observed variable in $\*O$ are parametrized by a family of neural networks, similar to \cite{louizosCausalEffectInference2017}. Using the computational framework of \emph{Generative Adversarial Networks} (GANs) \cite{goodfellow2014generative,nowozinFGANTrainingGenerative2016}, we obtain a model $\tilde{M} \in \2N_{\tuple{\G}}$ satisfying the observational constraints $P(\*o;\tilde{M}) = P(\*o)$. The imitating policy is trained through explicit interventions in the learned model $\tilde{M}$; a different GAN is then deployed to optimize the policy $\pi$ so that it imitates the observed trajectories drawn from $P(\*s)$.

\section{Experiments} \label{sec4}

\begin{figure}[t]
\hfill%
\begin{subfigure}{0.28\linewidth}\centering
\begin{tikzpicture}
      \def\outerr{3}
      \def\innerr{2.7}
      \node[vertex] (X) at (0, 0) {$X$};
      \node[vertex] (Z) at (1.5, 2) {$Z$};
      \node[vertex,opacity=0] (H1) at (0, -0.5) {$H$};
      \node[vertex,opacity=0] (H2) at (0, 2.1) {$H$};
      \node[uvertex] (L) at (2, 0.75) {$L$};
      \node[uvertex] (Y) at (3, 0) {$Y$};
      \node[vertex] (W) at (1, 0.75) {$W$};
      
      \draw[dir] (Z) -- (L);
      \draw[dir] (L) to [bend left=45] (X);
      \draw[dir] (Z) -- (Y);
      \draw[dir] (X) -- (Y);
      \draw[dir] (L) -- (W);
      \draw[dir] (Z) to (X);
  
      \draw[bidir] (Z) to [bend left=45] (Y);
      \draw[bidir] (W) to [bend right=45] (Y);
      \draw[bidir] (Z) to [bend left=0] (W);

      \begin{pgfonlayer}{back}
        \draw[fill=betterred!25,draw=none] \convexpath{Z, W, X}{\outerr mm};
        \node[circle,fill=betterred!65,draw=none,minimum size=2*\innerr mm] at (X) {};
        \node[circle,fill=betterblue!25,draw=none,minimum size=2*\outerr mm] at (Y) {};
        \node[circle,fill=betterblue!65,draw=none,minimum size=2*\innerr mm] at (Y) {};
    \end{pgfonlayer}
    \end{tikzpicture}
  \caption{Highway Driving}
  \label{fig3a}
  \end{subfigure}\hfill
\begin{subfigure}{0.325\linewidth}\centering
    \includegraphics[bb = 0 0 210 140, width=\linewidth]{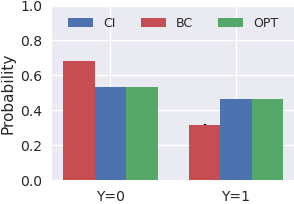}
\caption{Highway Driving}
\label{fig3b}
\end{subfigure}\hfill
\begin{subfigure}{0.325\linewidth}\centering
  \includegraphics[bb = 0 0 210 140, width=\linewidth]{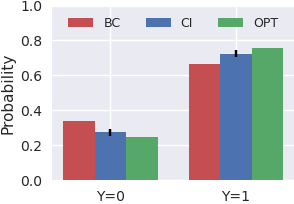} 
\caption{\texttt{MNIST} Digits}
\label{fig3c}
\end{subfigure}\hfill\null
\caption{(\subref{fig3a}) Causal diagram for highway driving example where a left-side car exists; (b,c) $P(y|\doo(\pi))$ induced by the causal imitation method (\textit{ci}) and the naive behavior cloning (\textit{bc}) compared with the actual distribution $P(y)$ over the expert's reward (\textit{opt}).}\label{fig3}
\end{figure}

We demonstrate our algorithms on several synthetic datasets, including \texttt{highD} \cite{highDdataset} consisting of natural trajectories of human driven vehicles, and on \texttt{MNIST} digits. In all experiments, we test our causal imitation method (\textit{ci}): we apply \Cref{thm:backdoor} when there exists an $\pi$-backdoor admissible set; otherwise, \Cref{imitate} is used to leverage the observational distribution. As a baseline, we also include na\"ive behavior cloning (\textit{bc}) that mimics the observed conditional distribution $P(x | \pa(\Pi))$, as well as the actual reward distribution generated by an expert (\textit{opt}). We found that our algorithms consistently imitate distributions over the expert's reward in imitable (p-imitable) cases; and p-imitable instances commonly exist. We refer readers to \cite[Appendix D]{appendix} for more experiments, details, and analysis.
    
\paragraph{Highway Driving} We consider a modified example of the drone recordings of human-driven cars in \Cref{sec1} where the driver's braking action $W$ of the left-side car is also observed. \Cref{fig3a} shows the causal diagram of this environment; $Z$ represent the velocity of the front-car; action $X$ represents the velocity of the driving car; $W$ and the reward signal $Y$ are both affected by an unobserved confounder $U$, representing the weather condition. In \Cref{fig3a}, $\set{Z}$ is $\pi$-backdoor admissible while $\set{Z, W}$ is not due to active path $X \leftarrow L \rightarrow W \leftrightarrow Y$. We obtain policies for the causal and naive imitators training two separate GANs. Distributions $P(y|\doo(\pi))$ induced by all algorithms are reported in \Cref{fig3b}. We also measure the L1 distance between $P(y|\doo(\pi))$ and the expert's reward $P(y)$. We find that the causal approach (\textit{ci}), using input set $\set{Z}$, successfully imitates $P(y)$ (L1 $= 0.0018$). As expected, the naive approach (\textit{bc}) utilizing all covariates $\set{Z, W}$ is unable to imitate the expert (L1 = $0.2937$).

\paragraph{MNIST Digits} We consider an instance of \Cref{fig1c} where $X, S, Y$ are binary variables; binary values of $W$ are replaced with corresponding images of \texttt{MNIST} digits (pictures of 1 or 0), determined based on the action $X$. For the causal imitator (\textit{ci}), we learn a POSCM $\hat{M}$ such that $P(x,w,s;\hat{M}) = P(x, w, s)$. To obtain $\hat{M}$, we train a GAN to imitate the observational distribution $P(x, w, s)$, with a separate generator for each $X, W, S$. We then train a separate discriminator measuring the distance between observed trajectories $P(s)$ and interventional distribution $P(s|\doo(\pi); \hat{M})$ over the surrogate $\set{S}$. The imitating policy is obtained by minimizing such a distance. Distributions $P(y|\doo(\pi))$ induced by all algorithms are reported in \Cref{fig3c}. We find that the causal approach (\textit{ci}) successfully imitates $P(y)$ (L1 $= 0.0634$). As expected, the naive approach (\textit{bc}) mimicking distribution $P(x)$ is unable to imitate the expert (L1 = $0.1900$).
\section{Conclusion}
We investigate the imitation learning in the semantics of structural causal models. The goal is to find an imitating policy that mimics the expert behaviors from combinations of demonstration data and qualitative knowledge about the data-generating process represented as a causal diagram. We provide a graphical criterion that is complete (i.e., sufficient and necessary) for determining the feasibility of learning an imitating policy that mimics the expert's performance. We also study a data-dependent notion of imitability depending on the observational distribution. An efficient algorithm is introduced which finds an imitating policy, by exploiting quantitative knowledge contained in the observational data and the presence of surrogate endpoints. Finally, we propose a practical procedure for estimating such an imitating policy from observed trajectories of the expert's demonstrations.

\section*{Broader Impact}
This paper investigates the theoretical framework of learning a policy that imitates the distribution over a primary outcome from natural trajectories of an expert demonstrator, even when the primary outcome itself is unobserved and input covariates used by the expert determining original values of the action are unknown. Since in practice, the actual reward is often unspecified and the learner and the demonstrator rarely observe the environment in the same fashion, our methods are likely to increase the progress of automated decision systems. Such systems may be applicable to various fields, including the development of autonomous vehicle, industrial automation and the management of chronic disease. These applications may have a broad spectrum of societal implications. The adoption of autonomous driving and industrial automation systems could save cost and reduce risks such as occupational injuries; while it could also create unemployment. Treatment recommendation in the clinical decision support system could certainly alleviate the stress on the healthcare workers. However, this also raise questions concerning with the accountability in case of medical malpractice; collection of private personal information could also make the hospital database valuable targets for malicious hackers. Overall, we would encourage research to understand the risks arising from automated decision systems and mitigations for its negative impact.

Recently, there is a growing amount of dataset of natural vehicle trajectories like \texttt{highD} \cite{highDdataset} being licensed for commercial use. An immediate positive impact of this work is that we discuss potential risk of training decision-making policy from the observational data due to the presence of unobserved confounding, as shown in \Cref{sec1,sec4}. More broadly, since our method is based on the semantics of structural causal models \cite[Ch.~7]{pearl:2k}, its adoption could cultivate machine learning practitioners with proper training in causal reasoning. A favorable characteristic of causal inference methods is that they are inherently robust: for example, the definition of imitability \Cref{def:imitability} requires the imitating policy to perfectly mimics the expert performance in \emph{any} model compatible to the causal diagram. Automated decision systems using the causal inference methods prioritize the safety and robustness in decision-making, which is increasingly essential since the use of black-box AI systems is prevalent and our understandings of their potential implications are still limited.

\section*{Acknowledge}
The authors were partially supported by grants from NSF IIS-1704352 and IIS-1750807 (CAREER).

\bibliographystyle{abbrv}

\clearpage
\appendix
\section{Proofs} \label{appendix:a}
\corolnonid*
\begin{proof}
Let $Y \in \*Y \cap \*L$. For any SCM $M_1$ that induces $\G$, we could obtain an SCM $M_2$ by replacing $f_Y$ and $P(u_Y)$ associated with $Y$. Since there is no restriction on the parametrization of $f_Y$ and $P(u_Y)$, we could always ensure $P(y| \doo(\pi); M_1) \neq P(y| \doo(\pi); M_2)$. For example, in both $M_1, M_2$, let $Y \gets U_Y$; we define $P(U_Y = 0;M_1) = 0.1$ and $P(U_Y = 0;M_2) = 0.9$. It is immediate to see that $P(\*y| \doo(\pi); M_1) \neq P(\*y| \doo(\pi); M_2)$, i.e., $P(\*y| \doo(\pi))$ is not identifiable.
\end{proof}

\lemnonimitable*
\begin{proof}
  The lack of existence of a shared imitating policy between $M_1, M_2$ eliminates the possibility of the existence of a function from $P(\*o)$ to a policy $\pi$ that imitates $P(\*y)$ in any POSCM compatible with the causal diagram $\G$.
\end{proof}

\thmdp*
\begin{proof}
  Since there is not bi-directed arrow pointing into $X$, it is variables that $\pa(X)_{\G}$ is $\pi$-backdoor admissible relative to $\tuple{X, Y}$ in $\G$. By Rule 2 of do-calculus, we have 
  \[
    P(y) = \sum_{x, \pa(X)_{\G}} P(\pa(X)_{\G}) P(x | \pa(X)_{\G}) P(y|\doo(x), \pa(X)_{\G}).  
  \]
  Since $\pa(X)_{\G} \subseteq \Pa(\Pi)$, the conditional distribution $P(x | \pa(X)_{\G})$ can be represented as a policy in $\Pi$. Let $\pi(x|\Pa(\Pi)) = \pi(x|\pa(X)_{\G}) = P(x | \pa(X)_{\G})$. We must have
  \[
    P(y) = \sum_{x, \pa(X)_{\G}} P(\pa(X)_{\G}) \pi(x | \pa(X)_{\G}) P(y|\doo(x), \pa(X)_{\G}) = P(y|\doo(\pi)). \qedhere
  \]
\end{proof}

\begin{lemma}\label{lem:imitable}
  Given a causal diagram $\G$ and a policy space $\Pi$, let $\*Z  = \An(Y)_{\G} \cap \Pa(\Pi)$. If $P(y)$ is imitable relative to $\tuple{\G, \Pi}$, then $\left (Y \ci X | \*Z \right)_{\G_{\underline{X}}}$.
\end{lemma}
\begin{proof}
  We consider first a simplified causal diagram $\1H$ where all endogenous variables are observed, i.e., $\*V = \*O$; and $\Pa(\Pi) = \*O \setminus \set{X, Y}$. In this diagram $\G$, $\left (Y \not \ci X | \*Z \right)_{\G_{\underline{X}}}$ implies that there exists a bi-directed path $l$ of the form $X \leftrightarrow Z_1 \leftrightarrow Z_2 \leftrightarrow \cdots \leftrightarrow Z_n \leftrightarrow Y$ such that $Z_i \in \*Z$ for any $i = 1, \dots, n$. We denote by $\*Z^* = \set{Z_1, \dots, Z_n}$. Recall that $X \in \An(Y)_{\1H}$. We could thus obtain a subgraph $\1H'$ of $\1H$ that satisfies the following condition:
  \begin{itemize}
    \item $\1H'$ contains the bi-directed path $l$.
    \item All nodes in $\1H'$ are descendants of $\*Z^* \cup \set{X}$.
    \item $Y$ is a descendant for all nodes in $\1H'$.
    \item Every endogenous node $V$ in $\1H'$ has at most one child.
    \item All bi-directed arrows in $\1H'$ are contained in $l$.
  \end{itemize}
  A mental image for depicting $\1H'$ is to think of a tree rooted in node $Y$; $X, Z_1, \dots, Z_n$ are leaf nodes; nodes $X, Z_1, \dots, Z_2, Y$ are connected by the bi-directed path $l$; $\1H'$ contain no bi-directed arrow except path $l$. We now construct a POSCM $M'$ compatible with $\G'$. More specifically, values of each exogenous confounder $U_i$ residing on $l$ are drawn uniformly over a binary domain $\set{0, 1}$. For each endogenous variable $V_i$ in $\1H'$, its values is equal to the parity sum of its parents in $\1H'$, i.e., $V_i \gets \oplus_{V_j \in \pa(V_i)_{\1H'}} V_j$. By construction, in the subgraph $\1H'$, each exogenous $U_i$ has exactly two directed paths going to $Y$. This means that in the constructed model $M'$, observed values of $Y$ is always equal to $0$, i.e., the reward distribution $P(Y = 0; M') = 1$.

  Consider a policy space $\Pi'$ of the form $\set{\pi: \D_{\*Z^*} \mapsto \PP_{X}}$. Let $U_n$ denote the exogenous confounder residing on $l$ that is closest to $Y$, i.e., $X \leftrightarrow Z_1 \leftrightarrow Z_2 \leftrightarrow \cdots \leftrightarrow Z_n \leftarrow U_n \rightarrow Y$. By the definition of $M'$, values of each variable $Z_i \in \*Z^*$ is the parity sum of at least two exogenous variables $U_j, U_k$. Since values of each $U_i$ are drawn uniformly over $\set{0, 1}$, we must have $P(u_n, \*z^*) = P(u_n)P(\*z^*)$, or equivalently, $P(u_n | \*z^*) = P(u_n)$. By definition, we could obtain from $l$ a directed path going from $U_n$ to $Y$ that is not intercepted by $\*Z^*$. That is, given any value $\*Z^* = \*z^*$, values of $Y$ is decided by a parity function taking $U_n$ as an input. Since $P(u_n | \*z^*) = P(u_n)$ and $U_n$ is drawn uniformly over $\set{0, 1}$, it is verifiable that in $M'$, given any $\*Z^* = \*z^*$, the conditional distribution $P(Y = 0|\doo(x), \*z^*; M') = 0.5$. This means that for any policy $\pi \in \Pi'$, the interventional distribution $P(Y = 0 | \doo(\pi); M') = 0.5$, which is far from the observational distribution $P(Y = 0; M') = 1$. That is, $P(y)$ is not imitable w.r.t. $\tuple{\1H', \Pi'}$.

  We will next show the non-imitability of $P(y)$ w.r.t. $\tuple{\1H, \Pi}$. For any node $V$ that is not included in $\1H'$, let its values be decided by an independent noise $U_{V}$ drawn uniformly over $\set{0, 1}$. We denote this extended POSCM by $M$. Obviously, $M$ is compatible with the causal diagram $\1H$. In this model, for any covariate $V \in \Pa(\Pi) \setminus \*Z^*$, it is either (1) a random variable disconnected to any other endogenous variable in $M$; or (2) decided by a function taking $\*Z^*$ as input. That is, $V \in \Pa(\Pi) \setminus \*Z^*$ contains no value of information with regard to the reward $Y$ when intervening on action $X$. By \cite[Ch.~23.6]{koller2009probabilistic}, this means that for any policy $\pi \in \Pi$, there exists a policy $\pi'$ in the subspace $\Pi'$ such that $P(y|\doo(\pi);M) = P(y|\doo(\pi');M)$. Recall that there exists no policy $\pi$ in $\Pi'$ that could ensure $P(y|\doo(\pi); M') = P(y; M')$ in POSCM $M'$. By definition of $M$ and $M'$, we must have $P(y;M) = P(y;M')$ and for any policy $\pi \in \Pi$, $P(y|\doo(\pi); M) = P(y|\doo(\pi); M')$. It is immediate to see that there exists no policy $\pi \in \Pi$ that could induce $P(y|\doo(\pi); M) = P(y; M)$ in the extended POSCM $M$, i.e., $P(y)$ is not imitable w.r.t. $\tuple{\1H, \Pi}$.

  We now consider a general causal diagram $\G$ where arbitrary latent endogenous variables $\*L$ exist and $\Pa(\Pi) \subset \*O \setminus \set{X, Y}$. We will apply the latent projection \cite[Def.~5]{tian:02} to transforms $\G$ into a simplified causal diagram $\1H$ discussed above. More specifically, we construct a causal diagram $\G'$ from $\G$ by marking each $V \in \*V \setminus (Pa(\Pi) \cup \set{X, Y})$ latent. We then apply \Cref{project} and obtain a simplified diagram $\1H = \textsc{Project}(\G')$ where $\*V = \*O = \Pa(\Pi) \cup \set{X, Y}$. Since \textsc{Project} preserves topological relationships among observed nodes \cite[Lem.~5]{tian:02}, $(Y \not \ci X | \*Z)_{\G_{\underline{X}}}$ implies $(Y \not \ci X | \*Z)_{\1H_{\underline{X}}}$. Following our previous argument, there exists a POSCM $M'$ associated with $\1H$ such that for any policy $\pi \in \Pi$, $P(y|\doo(\pi); M') \neq P(y; M')$. By \Cref{lem:proj2}, we could construct a POSCM $M$ associated with $\G$ such that for any $\pi \in \Pi$, $P(y|\doo(\pi); M) = P(y|\doo(\pi); M')$ and $P(y; M) = P(y; M')$. It follows immediately that for any policy $\pi \in \Pi$, $P(y|\doo(\pi); M) \neq P(y; M)$, i.e., $P(y)$ is not imitable w.r.t. $\tuple{\G, \Pi}$.
\end{proof}

\thmbackdoor*
\begin{proof}
  We first prove the the ``if'' direction. Given an $\pi$-backdoor admissible set $\*Z$ relative to $\tuple{X, Y}$ in $\G$. By Rule 2 of do-calculus, we have,
  \[
    P(y) = \sum_{x, \*z} P(\*z) P(x | \*z) P(y|\doo(x), \*z).  
  \]
  Since $\*Z \subseteq \Pa(\Pi)$, the conditional distribution $P(x | \*z)$ can be represented as a policy in $\Pi$. Let $\pi(x|\Pa(\Pi)) = \pi(x|\*z) = P(x | \*z)$. We must have
  \[
    P(y) = \sum_{x, \*z} P(\*z) \pi(x | \*z) P(y|\doo(x), \*z) = P(y|\doo(\pi)).
  \]
  We now consider the ``only if'' direction. Suppose there exists no $\pi$-backdoor admissible set $\*Z$ relative to $\tuple{X, Y}$ in $\G$. We must have that $\*Z = \An(Y)_{\G} \cap \Pa(\Pi)$ is not $\pi$-backdoor admissible, i.e., the independent relationship $\left (Y \ci X | \*Z \right)_{\G_{\underline{X}}}$ does not hold. It follows immediately from \Cref{lem:imitable} that $P(y)$ is not imitable relative to $\tuple{\G, \Pi}$.
\end{proof}

\lemsurrogate*
\begin{proof}
  Since $\*S$ is a surrogate relative to $\tuple{\G, \Pi'}$, by definition, we have $(Y \ci \hat{X} | \*S)_{\G \cup \Pi'}$. For any causal diagram $\G$ and policy space $\Pi$, let $\G_{\Pi}$ denote a manipulated diagram obtained from $\G$ by adding arrows from nodes in $\Pa(\Pi)$ to $X$ in the subgraph $\G_{\overline{X}}$. By definition, it is obvious that  $\G \cup \Pi'$ is a supergraph containing both $\G$ and $\G_{\Pi'}$. By definition of d-separation, we must have $(Y \ci \hat{X} | \*S)_{\G}$ and $(Y \ci \hat{X} | \*S)_{\G_{\Pi'}}$. The basic operations of distribution marginalization implies, for any POSCM $M \in \2M_{\langle \G, P \rangle}$, 
\begin{align*}
    P(y| \doo(\pi); M) &= \sum_{\*s} P(y|\*s, \doo(\pi); M)P(\*s | \doo(\pi); M)\\
    &=\sum_{\*s} P(y|\*s, \doo(x); M)P(\*s| \doo(\pi); M)\\
    &=\sum_{\*s} P(y|\*s; M)P(\*s|\doo(\pi); M)
\end{align*}
The last two steps hold since $(Y \ci \hat{X} | \*S)_{\G_{\Pi'}}$ and $(Y \ci \hat{X} | \*S)_{\G}$. Fix an arbitrary POSCM $M \in \2M_{\langle \G, P \rangle}$. Suppose there exists an imitating policy $\pi \in \Pi'$ in $M$ such that $P(\*s|\doo(\pi);M) = P(\*s;M)$. We must have
\begin{align}
    P(y|\doo(\pi);M) = \sum_{\*s} P(y|\*s;M)P(\*s|\doo(\pi); M) = \sum_{\*s} P(y|\*s;M)P(\*s;M)= P(y;M). \label{eq:surrogate1}
\end{align}
Since $\tuple{\G, \Pi'}$ is an instrument w.r.t. $\tuple{\G, \Pi}$, $P(\*s|\doo(\pi))$ is identifiable w.r.t. $\tuple{\G, \Pi'}$. This means that $P(\*s|\doo(\pi);M)$ is uniquely computable from $P(\*o;M)$ in any POSCM $M \in \2M_{\langle \G \rangle}$. That is, for any POSCM $M \in \2M_{\langle \G, P \rangle}$ where its observational distribution $P(\*o;M) = P(\*o)$, the interventional distribution $P(\*s|\doo(\pi);M)$ for any policy $\pi \in \Pi'$ remains as an invariant. This implies that the derivation in \Cref{eq:surrogate1} is applicable for any POSCM $M \in 
\2M_{\langle \G, P \rangle}$, i.e., $P(y)$ is p-imitable w.r.t. $\tuple{\G, \Pi, \theta}$.
\end{proof}

\lemidY*
\begin{proof}
  Let $\*S$ be a surrogate w.r.t. $\tuple{\G, \Pi'}$. Following the proof of \Cref{lem:surrogate},
  \begin{align*}
    P(y| \doo(\pi)) &= \sum_{\*s} P(y|\*s, \doo(\pi))P(\*s | \doo(\pi))\\
    &=\sum_{\*s} P(y|\*s, \doo(x))P(\*s| \doo(\pi))\\
    &=\sum_{\*s} P(y|\*s)P(\*s|\doo(\pi)).
  \end{align*}
  Suppose $\tuple{\*S, \Pi'}$ form an instrument w.r.t. $\tuple{\G, \Pi}$, i.e., $P(\*s|\doo(\pi))$ is identifiable w.r.t. $\tuple{\G, \Pi'}$. The above equation implies that $P(y|\doo(\pi))$ can be uniquely determined from $P(\*o, y)$ in any POSCM that induces $\G$. That is, $P(y|\doo(\pi))$ is identifiable w.r.t. $\tuple{\G \cup \set{Y}, \Pi}$, which completes the proof.
\end{proof}

\lemminimal*
\begin{proof}
  For any POSCM $M$ in $\2M_{\langle \G, P \rangle}$, if there exists a policy $\pi \in \Pi$ such that $P(\*s|\doo(\pi);M)=P(\*s;M)$, we must have for any $\*S' \subseteq \*S$, 
  \[
    P(\*s'|\doo(\pi);M)= \sum_{\*s \setminus \*s'} P(\*s|\doo(\pi);M) = \sum_{\*s \setminus \*s'} P(\*s;M) = P(\*s;M).
  \]
  which completes the proof.
\end{proof}

\thmpimitation*
\begin{proof}
  For a policy subspace $\Pi'$, Step 3 outputs a surrogate $\*S$ w.r.t. $\tuple{\G, \Pi'}$. Step 4 ensures that $\Pi'$ is an id-subspace w.r.t. $\tuple{\G, \Pi, \*S}$. That is, $\tuple{\*S, \Pi'}$ forms an instrument w.r.t. $\tuple{\G, \Pi}$. Since \textsc{Imitate} only outputs a policy $\pi$ when $P(\*s)$ is p-imitable w.r.t. $\tuple{\G, \Pi', P(\*o)}$, the statement follows from \Cref{lem:surrogate}.
\end{proof}
\clearpage
\section{Causal Identification in POSCMs} \label{appendix:b}
In this section, we introduce algorithms for identifying causal effects in Partially Observable Structural Causal Models (POSCMs), which are sufficient and complete. We will consistently assume that $\*Y \subseteq \*O$, i.e., the primary outcomes $\*Y$ are all observed. For settings where $\*Y \not \subseteq \*O$, \Cref{corol:nonid} implies that $P(\*y|\doo(\pi))$ is always non-identifiable w.r.t. the causal diagram $\G$. For convenience, we focus on the problem of determining whether the target effect is identifiable w.r.t. $\G$. However, our algorithms could be easily extended to derive identification formulas of the causal effect. 

We start with the identificaiton of causal effects induced by atomic interventions $\doo(x)$. Formally,
\begin{definition}[Identifiability (Atomic Interventions)]\label{def:atomic_id}
    Given a causal diagram $\G$, let $\*Y$ be an arbitrary subset of $\*O$. $P(\*y|\doo(x))$ is said to be identifiable w.r.t. $\1G$ if $P\left(\*y|\doo(x); M \right)$ is uniquely computable from $P(\*o; M)$ and $\pi$ for any POSCM $M \in \2M_{\langle \G \rangle}$.
\end{definition}
A causal diagram $\G$ is said to be semi-Markovian if it does not contain any latent endogenous variables $\*L$; or equivalently, $\*V = \*O$. \cite[Alg.~5]{tian:02} is an algorithm that identify causal effects, say $P(\*y|\doo(\*x))$, from the observational distribution $P(\*o)$ in a semi-Markovian diagram $\G$, which we consistently refer to as \textsc{IdentifyHelper}. More specifically, \textsc{IdentifyHelper} takes as input a set of action variables $\*X \subseteq \*O$, a set of outcome variables $\*Y \subseteq \*O$ and a semi-Markovian causal diagram $\G$\footnote{In the original text, the semi-Markovian causal diagram $\G$ is implicitly assumed; \cite[Alg.~5]{tian:02} only takes $\*X, \*Y$ as input. We rephrase the algorithm and explicitly represent the dependency on the causal diagram $\G$.}. If $P(\*y|\doo(\*x))$ is identifiable w.r.t. $\G$, \textsc{IdentifyHelper} returns an identificaiton formula that represents $P(\*y|\doo(\*x))$ as an algebraic expression of the observational distribution $P(\*o)$; otherwise, \textsc{IdentifyHelper} returns ``\textsc{Fail}''. \cite{huang:val06,shpitser:pea06a} showed that \textsc{IdentifyHelper} is complete from identifying effects from observational data with respect to semi-Markovian causal diagrams.

We will utilize \textsc{IdentifyHelper} to identifying causal effects in POSCMs where latent endogenous variables are allowed. The key to this reduction is an algorithm \textsc{Project} \cite[Def.~5]{tian:02} that transforms an arbitrary causal diagram $\G$ with observed endogenous variables $\*O$ into a semi-Markovian causal diagram $\1H$ such that its endogenous variables $\*V = \*O$. For completeness, we rephrase \textsc{Project} and describe it in \Cref{project}.
\begin{algorithm}[h]
    \caption{\textsc{Project} \cite[Def.~5]{tian:02}}
    \label{project}
  \begin{algorithmic}[1]
    \STATE {\bfseries Input:} A causal diagram $\G$.
    \STATE {\bfseries Output:} A causal diagram $\1H$ where all endogenous variables are observed, i.e., $\*V = \*O$.
    \STATE Let $\*O, \*L$ be, respectively, observed endogenous variables and latent endogenous variables in $\G$.
    \STATE Let $\1H$ be a causal diagram constructed as follows.
    \FOR{each observed $V \in \*O$ in $\G$}
    \STATE Add an observed node $V$ in $\1H$.
    \ENDFOR
    \FOR{each pair $S, E \in \*O$ in $\G$ s.t. $S \neq E$}
    \IF{there exists a directed path $S \rightarrow E$ in $\G$}
    \STATE Add an edge $S \rightarrow E$ in $\1H$.
    \ELSIF{there exists a path $S \rightarrow V_1 \rightarrow \cdots \rightarrow V_n \rightarrow E$ in $\G$ s.t. $V_1, \cdots V_n \in \*L$}
    \STATE Add an edge $S \rightarrow E$ in $\1H$.
    \ELSIF{there exists a bidirected edge $S \leftrightarrow E$ in $\G$}
    \STATE Add a bidirected edge $S \leftrightarrow E$ in $\1H$.
    \ELSIF{there exists a path $S \leftarrow V_{l,1} \leftarrow \cdots \leftarrow V_{l,n} \leftrightarrow V_{r,m} \rightarrow \cdots \rightarrow V_{r,1} \rightarrow E$ in $\G$ s.t. $V_{l,1}, \cdots, V_{l, n}, V_{r, 1}, \cdots, V_{r,m} \in \*L$}
    \STATE Add a bidirected edge $S \leftrightarrow E$ in $\1H$.
    \ELSIF{there exists a path $S \leftarrow V_{l,1} \leftarrow \cdots \leftarrow V_{l,n} \leftarrow V_c \rightarrow V_{r,m} \rightarrow \cdots \rightarrow V_{r,1} \rightarrow E$ in $\G$ s.t. $V_{l,1}, \cdots, V_{l, n}, V_c, V_{r, 1}, \cdots, V_{r,m} \in \*L$}
    \STATE Add a bidirected edge $S \leftrightarrow E$ in $\1H$.
    \ENDIF
    \ENDFOR
    \STATE Return $\1H$.
  \end{algorithmic}
\end{algorithm}

The following lemma, introduced in \cite[Props.~2-3]{lee2019structural}, shows that the parameter space of observational distributions $P(\*o)$ and interventional distributions $P(\*y|\doo(x))$ induced by POSCMs associated with a causal diagram $\G$ is always equivalent to that induced by the corresponding semi-Markovian causal diagram $\1H = \textsc{Project}(\G)$.
\begin{lemma}\label{lem:proj}
    Given a causal diagram $\G$, let $\1H = \textsc{Project}(\G)$. For any POSCM $M_1$ associated with $\G$, there exists a POSCM $M_2$ associated with $\1H$ such that $P(\*y|\doo(\*x); M_1) = P(\*y|\doo(\*x);M_2)$ for any $\*X, \*Y \subseteq \*O$, and vice versa.
\end{lemma}
\begin{proof}
    The statement follows from \cite[Props.~2-3]{lee2019structural}.
\end{proof}
\Cref{lem:proj} implies a general algorithm for identifying $P(\*y|\doo(x))$ in a causal diagram $\G$ with latent endogenous variables: it is sufficient to consider the identifiability of $P(\*y|\doo(x))$ in the projection $\1H = \textsc{Project}(\G)$. We describe such an algorithm in \Cref{identify-atomic}.
\begin{algorithm}[h]
      \caption{\textsc{Identify} (Atomic Interventions)}
      \label{identify-atomic}
    \begin{algorithmic}[1]
      \STATE {\bfseries Input:} A causal diagram $\G$, primary outcomes $\*Y \subseteq \*O$.
      \STATE Let $\1H = \textsc{Project}(\G)$.
      \IF{$\textsc{IdentifyHelper}(\set{X}, \*Y, \1H) = \textsc{Fail}$} 
      \STATE Return \textsc{No}.
      \ELSE 
      \STATE Return \textsc{Yes}.
      \ENDIF
    \end{algorithmic}
\end{algorithm}
\begin{corollary}\label{corol:id1}
    Given a causal diagram $\G$, let $\*Y$ be an arbitrary subset of $\*O$. $\textsc{Identify}(\G, \*Y) = ``\textsc{Yes}''$ if and only if $P(\*y|\doo(x))$ is identifiable w.r.t. $\G$.
\end{corollary}
\begin{proof}
    \Cref{lem:proj} implies that $P(\*y|\doo(x))$ is identifiable w.r.t. a causal diagram $\G$ if and only if $P(\*y|\doo(x))$ is identifiable w.r.t. the projection $\1H = \textsc{Project}(\G)$. To see this, suppose $P(\*y|\doo(x))$ is not identifiable w.r.t. $\G$. That is, there exist two POSCMs $M_1, M_2$ associated with $\G$ such that $P(\*o;M_1) = P(\*o;M_2)$ while $P(\*y|\doo(x);M_1) \neq P(\*y|\doo(x);M_2)$. By \Cref{lem:proj}, for any $M_i$ and $i = 1, 2$, we could find a POSCM $M'_i$ associated with $\1H$ such that $P(\*o;M_i) = P(\*o;M'_i)$ and $P(\*y|\doo(x);M_i) = P(\*y|\doo(x);M'_i)$. This implies that we could obtain two POSCMs $M'_1, M'_2$ associated with $\1H$ such that $P(\*o;M'_1) = P(\*o;M'_2)$ while $P(\*y|\doo(x);M'_1) \neq P(\*y|\doo(x);M'_2)$, i.e., $P(\*y|\doo(x))$ is identifiable w.r.t. $\1H$. Similarly, we could prove the ``only if'' direction.

    Since \textsc{Identify} returns ``\textsc{Yes}'' if and only if \textsc{IdentifyHelper} finds an identification formula of $P(\*y|\doo(x))$ and \textsc{IdentifyHelper} is sound and complete, the statement is entailed.
\end{proof}

\subsection{Identifying Conditional Plans}
We will next study the general problem of identifying causal effects $P(\*y|\doo(\pi))$ induced by interventions $\doo(\pi)$ following conditional plans in a policy space $\Pi$. The formal definition of such identifiability is given in \Cref{def:id}. Similar to atomic interventions, we consider first a simpler setting where the causal diagram $\G$ is semi-Markovian, without latent endogenous variables. Given a causal diagram $\G$, we denote by $\G_{\Pi}$ a manipulated diagram obtained from a subgraph $\G_{\overline{X}}$ by adding arrows from nodes in $\Pa(\Pi)$ to the action node $X$. Let a set $\*Z = \an(\*Y)_{\G_{\Pi}} \setminus \set{X}$. Following \cite[Eq.~15]{tian2008identifying}, the interventional distribution $P(\*y|\doo(\pi))$ could be written as follows:
\begin{align}
    P(\*y|\doo(\pi)) &= \sum_{x, \*z} P(\*y, \*z| \doo(x), \doo(\*v \setminus (\*y \cup \*z \cup \set{x}))) \pi(x|\pa(\Pi)) \notag \\
    &= \sum_{x, \*z} P(\*y, \*z| \doo(x)) \pi(x|\pa(\Pi)). \label{eq:b1}
\end{align}
Among the above equations, the last step follows from Rule 3 of do-calculus \cite[Thm.~3.4.1]{pearl:2k}. More specifically, if there is a directed path from a node $V_i \in \*V \setminus (\*Y \cup \*Z \cup \set{X})$ to a node in $\*Y, \*Z$ in the subgraph $\G_{\overline{X}}$, $V_i$ must also be included in set $\*Z$. That is, $\left (\*Y, \*Z \ci \*V \setminus (\*Y \cup \*Z \cup \set{X}) \right)_{\G_{\overline{X}, \overline{\*V \setminus (\*Y \cup \*Z \cup \set{X})}}}$, which implies $P(\*y, \*z| \doo(x), \doo(\*v \setminus (\*y \cup \*z \cup \set{x}))) =  P(\*y, \*z| \doo(x))$. It is thus sufficient to identify the causal effect $P(\*y, \*z | \doo(x))$ induced by atomic intervention $\doo(x)$ w.r.t. $\G$. \cite{shpitser2018identification,correa2019statistical} showed that such an algorithm is complete for identifying $P(\*y|\doo(\pi))$ in a semi-Markovian causal diagram $\G$ where all endogenous variables are observed. 
\begin{lemma}\label{lem:cond-plan}
    Given a causal diagram $\G$ and a policy space $\Pi$, let $\*Y$ be an arbitrary subset of $\*O$. Assume that $\G$ is semi-Markovian, i.e., $\*V  = \*O$. Let $\*Z = \an(\*Y)_{\G_{\Pi}} \setminus \set{X}$. $P(\*y|\doo(\pi))$ is identifiable w.r.t. $\tuple{\G, \Pi}$ if and only if $P(\*y, \*z| \doo(x))$ is identifiable w.r.t. $\G$.
\end{lemma}
\begin{proof}
    The statement follows immediately from  \cite[Corol.~2]{correa2019statistical}.
\end{proof}
We are now ready to consider the identificaiton of $P(\*y|\doo(\pi))$ w.r.t. a policy space $\Pi$ and a general causal diagram $\G$ where latent endogenous variables $\*L$ are present. Our next result shows that it is sufficient to identify $P(\*y|\doo(\pi))$ in the corresponding semi-Markovian projection $\1H = \textsc{Project}(\G)$.  
\begin{lemma}\label{lem:proj2}
    Given a causal diagram $\G$ and a policy space $\Pi$, let $\1H = \textsc{Project}(\G)$. For any POSCM $M_1$ associated with $\G$, there exists a POSCM $M_2$ associated with $\1H$ such that $P(\*y|\doo(\pi); M_1) = P(\*y|\doo(\pi);M_2)$ for any $\pi \in \Pi$, any $\*Y \subseteq \*O$, and vice versa.
\end{lemma}
\begin{proof}
    For any policy $\pi \in \Pi$, $P(\*y|\doo(\pi))$ could be written as a function of $P(\*y, \*z|\doo(x))$ and $\pi$ following \Cref{eq:b1}. Therefore, the statement is implied by \Cref{lem:proj}.
\end{proof} 
\begin{algorithm}[h]
    \caption{\textsc{Identify}}
    \label{identify}
  \begin{algorithmic}[1]
    \STATE {\bfseries Input:} A causal diagram $\G$, a policy space $\Pi$, primary outcomes $\*Y \subseteq \*O$.
    \STATE Let $\1H = \textsc{Project}(\G)$.
    \STATE Let $\*Z = \an(\*Y)_{\1H_{\Pi}} \setminus \set{X}$.
    \IF{$\textsc{IdentifyHelper}(\set{X}, \*Y \cup \*Z, \1H) = \textsc{Fail}$} 
    \STATE Return \textsc{No}.
    \ELSE 
    \STATE Return \textsc{Yes}.
    \ENDIF
  \end{algorithmic}
\end{algorithm}
Details of our algorithm \textsc{Identify} is described in \Cref{identify}. It takes as input a causal diagram $\G$, a policy space $\Pi$ and a set of observed outcomes $\*Y$. At Step 1, it obtains a projection $\1H$ of $\G$ such that all endogenous variables $\*V$ in $\1H$ are observed. It then constructs the covariates $\*Z$ following \Cref{lem:cond-plan} (Step 3). Finally, \textsc{Identify} calls \textsc{IdentifyHelper} to identify $P(\*y, \*z|\doo(x))$ in the projection $\1H$. It outputs ``\textsc{No}'' if \textsc{IdentifyHelper} fails to find an identification formula of $P(\*y, \*z|\doo(x))$; otherwise, it outputs ``\textsc{Yes}''. Since \Cref{lem:proj2} shows that the parameter space of $P(\*y|\doo(\pi))$ and $P(\*o)$ induced by POSCMs $M \in \2M_{\langle \G \rangle}$ is equivalent to that induced by instances in the family $\2M(\1H)$ of projection $\1H$, the soundness and completeness of \textsc{Identify} is entailed.
\begin{corollary}\label{corol:id2}
    Given a causal diagram $\G$ and a policy space $\Pi$, let $\*Y$ be an arbitrary subset of $\*O$. $\textsc{Identify}(\G, \Pi, \*Y) = \textsc{Yes}$ if and only if $P(\*y|\doo(\pi))$ is identifiable w.r.t. $\tuple{\G, \Pi}$.
\end{corollary}
\begin{proof}
    \Cref{lem:proj2} implies that $P(\*y|\doo(\pi))$ is identifiable w.r.t. $\tuple{\G, \Pi}$ if and only if $P(\*y|\doo(\pi))$ is identifiable w.r.t. $\tuple{\1H, \Pi}$ where $\1H = \textsc{Project}(\G)$. The proof is similar to \Cref{corol:id1}. It follows from \Cref{eq:b1} and \Cref{lem:cond-plan} that \textsc{IdentifyHelper} does not fail if and only if $P(\*y|\doo(\pi))$ is identifiable w.r.t. $\tuple{\1H, \Pi}$. The soundness and completeness of \textit{Identify} is thus entailed.
\end{proof}
\clearpage
\section{\textsc{ListIdSpace}}\label{appendix:c}
In this section, we describe Algorithm \textsc{ListIdSpace} that finds all identifiable subspaces with respect to a causal diagram $\G$, a policy space $\Pi$ and a set of observed variables $\*Y \subseteq \*O$. The details of \textsc{ListIdSpace} are shown in \Cref{listidspace}. 
\begin{figure}[h]
  \centering
  \begin{minipage}{\textwidth}
  \begin{algorithm}[H]
    \caption{\textsc{ListIdSpace}}
    \label{listidspace}
  \begin{algorithmic}[1]
   \STATE {\bfseries Input:} $\G, \Pi, \*Y$.
   \STATE $\textsc{ListIdSpaceHelper}(\G, \*Y, \set{\pi: \PP_{X}}, \Pi)$.
  \end{algorithmic}
  \end{algorithm}
  \end{minipage}
  \begin{minipage}{\textwidth}
  \begin{algorithm}[H]
    \caption{\textsc{ListIdSpaceHelper}}
    \label{listidspacehelper}
  \begin{algorithmic}[1]
    \STATE {\bfseries Input:} $\G, \*Y$ and policy spaces $\Pi_L, \Pi_R$ such that $\Pi_L \subseteq \Pi_R$.
    \IF{$\textsc{Identify}(\G, \Pi_L, \*Y) = \textsc{Yes}$}
    \IF{$\LL = \RR$} Output $\Pi_L$. 
    \ELSE 
    \STATE Pick an arbitrary $V \in \Pa(\Pi_R) \setminus \Pa(\Pi_L)$.
    \STATE $\textsc{ListIdSpaceHelper}(\G, \*Y, \Pi_L \cup \set{V}, \Pi_R)$.
    \STATE $\textsc{ListIdSpaceHelper}(\G, \*Y, \Pi_L, \Pi_R \setminus \set{V})$.
    \ENDIF
    \ENDIF
  \end{algorithmic}
  \end{algorithm}
  \end{minipage}
\end{figure}

It calls a subroutine \textsc{ListIdSpaceHelper} which takes as input the diagram $\G$, variables $\*Y$ and two policy subspace $\Pi_L, \Pi_R$ of $\Pi$ such that $\Pi_L \subseteq \Pi_R$. More specifically, \textsc{ListIdSpaceHelper} performs backtrack search to enumerate identifiable subspaces $\Pi'$ w.r.t. $\tuple{\G, \Pi, \*Y}$ such that $\Pi_L \subseteq \Pi' \subseteq \Pi_R$. It aborts branches that will not lead to such identifiable subspaces. The aborting criterion (Step 2 of \Cref{listidspacehelper}) follows the observation that $P(\*y|\doo(\pi))$ is identifiable w.r.t. a policy space $\Pi$ only if it is identifiable w.r.t. any subspace $\Pi' \subseteq \Pi$. At Step 5, it picks an arbitrary variable $V$ that is included in the input covariates of $\Pi_R$ but not in $\Pi_L$. Let $\Pi \cup \set{V}$ denote a policy space obtained from $\Pi$ by including $V$ as part of the input covariates, i.e., $\Pi \cup \set{V} = \set{\pi: \D_{\Pa(\Pi), V} \mapsto \PP_X}$. Similarly, we define $\Pi \cup \set{V} = \set{\pi: \D_{\Pa(\Pi) \setminus \set{V} } \mapsto \PP_X}$. \textsc{ListIdSpaceHelper} then recursively returns all identifiable subspaces $\Pi'$ w.r.t. $\tuple{\G, \Pi, \*Y}$: the first recursive call returns i-subspaces taking $V$ as an input and the second call return all i-subspaces that does not consider $V$.
\begin{restatable}{theorem}{thmlistidspace}\label{thm:listidspace}
  Given a causal diagram $\G$, a policy space $\Pi$, and an oracle access to \textsc{Identify}, let $\*Y$ be an arbitrary subset of $\*O$. $\textsc{ListIdSpace}(\G, X, \*Y, \emptyset, \Pa(\Pi))$ enumerates identifiable subspaces w.r.t. $\tuple{G, \Pi, \*Y}$ with polynomial delay $\mathcal{O}(|\Pa(\Pi)|)$.
\end{restatable}
\begin{proof}
    The recursive calls at Steps 6 and 7 guarantees that \textsc{ListIdSpaceHelper} generates every i-subspaces $\Pi'$ exactly once. Since every leaf will output an i-subspace, the tree height is at most $|\Pa(\Pi)|$ and the existence check is performed by \textsc{Identify} oracle, the delay time is $\mathcal{O}(|\Pa(\Pi)|)$.
\end{proof}
\clearpage
\section{Experiments}\label{appendix:d}

We perform 4 experiments, each designed to test different aspects of causal imitation learning. The basic results of the last two experiments are summarized in the paper's main text.

\begin{enumerate}
    \item \textbf{Data-Dependent Binary Imitability} - given a binary model, we show that by sampling uniformly from binary distributions satisfying the graph in figure \cref{fig1c} (called the ``frontdoor"), na\"ive imitation results in a biased answer. We also observe that in binary frontdoor models, 50\% of distributions exhibit data-dependent imitability, despite not being imitable in generality.
    \item \textbf{GAN-Based Binary Imitability} - Generative Adversarial Networks \cite{goodfellow2014generative} allow explicit parameterization of a model, and can therefore be used to imitate causal mechanisms in the presence of confounding. We show in various binary models that a GAN-based approach to imitation leads to positive results, both with standard and data-dependent imitation.
    \item \textbf{Highway Driving} - The binary models in experiment 2 show that GANs function in the causal imitation setting, however, one can simply use the corresponding formulae to efficiently get answers. In this experiment, we use the HighD dataset to make two variables continuous, and to show that na\"ively choosing information to use for imitation can lead to bias.
    \item \textbf{MNIST Digits} - The final experiment shows that the GAN-based approach is capable of handling complex, high-dimensional probability distributions. To show this, we perform data-dependent imitation on a frontdoor graph, replacing a node with pictures of MNIST digits, to represent a complex, high-dimensional probability distribution.
\end{enumerate}

Experiments 3 and 4 are reported in the main text. However, each experiment builds upon the ideas introduced in the preceding experiment, so it is recommended that readers go in order. Details of each experiment are described in its own subsection. For an discrete variable $X$, we will consistently use $x_i$ to represent an assignment $X = i$; therefore, we write $P(y_1|\doo(x_0)) = P(Y = 1|\doo(X = 0))$.

\subsection{Data-Dependent Binary Imitability}

If all variables in a causal model are discrete, one can find the imitating policy through the solution of a series of linear systems. As an example, in the front-door graph (\cref{fig1c}, and shown below), with all variables binary, and given an observational distribution $P(x, w, s)$ that is amenable to data-dependent imitation, the imitating policy $\pi$ for $X$ has probability of $\pi(x_1) = \alpha$:
\begin{figure}[H]
\begin{subfigure}{0.3\linewidth}\centering
\begin{tikzpicture}
    \def\outerr{3}
    \def\innerr{2.7}
    \node[vertex] (X) at (0, 0) {$X$};
      \node[vertex] (W) at (0, -1) {$W$};
      \node[vertex] (S) at (1, -1) {$S$};
      \node[uvertex] (Y) at (2, -1) {$Y$};
      
      \draw[dir] (X) -- (W);
      \draw[dir] (W) -- (S);
      \draw[dir] (S) -- (Y);

      \draw[bidir] (X) to [bend left=30] (S);

      \begin{pgfonlayer}{back}
        \draw[fill=betterblue!25,draw=none] \convexpath{S, Y}{\outerr mm};
        \node[circle,fill=betterred!25,draw=none,minimum size=2*\outerr mm] at (X) {};
        \node[circle,fill=betterred!65,draw=none,minimum size=2*\innerr mm] at (X) {};
        \node[circle,fill=betterblue!25,draw=none,minimum size=2*\outerr mm] at (Y) {};
        \node[circle,fill=betterblue!65,draw=none,minimum size=2*\innerr mm] at (Y) {};
    \end{pgfonlayer}
  \end{tikzpicture}
\end{subfigure}
\begin{subfigure}{.65\linewidth}
$$
\alpha P(s|\doo(x_1)) + (1-\alpha)P(s|\doo(x_0)) = P(s)
 $$
 \begin{equation}
\label{eq:frontdoor_alpha}
 \ \ \Rightarrow \alpha = \frac{P(s)-P(s|\doo(x_0))}{P(s|\doo(x_1)) -P(s|do(x=0)) }
\end{equation}
\end{subfigure}
\end{figure}

In this experiment, we show the bias of cloning $\pi(x)=P(x)$ as compared to imitation using \cref{eq:frontdoor_alpha} in binary models. It is important to note that we limited this experiment to binary models to permit the computation of optimal policies explicitly - higher dimensional/continuous models are likely to show different average bias due to their extra degrees of freedom.

\begin{figure}[ht]
    \centering
    \begin{subfigure}{0.45\textwidth}
    \centering
     \includegraphics[width=.8\linewidth]{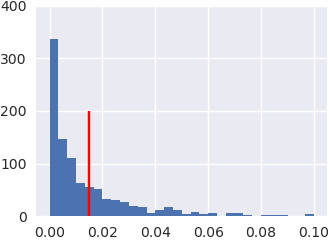}
     \caption{\textit{bc}}
    \end{subfigure}
    \begin{subfigure}{0.45\textwidth}
    \centering
     \includegraphics[width=.8\linewidth]{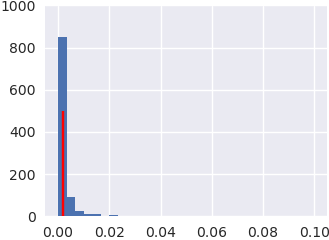}
     \caption{\textit{ci}}
    \end{subfigure}
    
     \caption{When the models are sampled uniformly from probability distributions compatible with \cref{fig1c}, and for which there exists a perfect policy for imitating $P(y)$, a na\"ive cloning of $\pi(x) = P(x)$ (a) yields a biased answer, while our approach (b) gets the answer with only sampling error.}
    \label{fig:ddbinbias}
    
\end{figure}

We generate $1\times 10^5$ random instances where $X, W, S$ are binary variables. Probability distributions consistent with the graph decompose as $P(x,w,s,y)=P(x)P(w|x)P(s|x, w)P(y|s)$, so we draw uniformly over $[0,1]$ for each conditional probability (i.e. $P(x),P(w|x),P(s|x, w), P(y|s) \sim U(0,1)$)
We report in \Cref{fig:ddbinbias} the L1 distance between the interventional distribution $P(y|\doo(\pi))$ induced by the imitator (\textit{ci} and \textit{bc}) and the actual expert's reward distribution $P(s)$ over $1\times 10^5$ generated instances. The causal imitiation learning approach uses \cref{eq:frontdoor_alpha}, while na\"ive cloning directly imitates $P(X)$, both using sample averages. We find that the causal imitation (\textit{ci}, average L1 = $0.0016$) dominates the na\"ive cloning approach (\textit{bc}, average L1 = $0.0147$). More interestingly, we find $50\%$ of generated instances are p-imitable; this suggests that leveraging the observational distribution is beneficial in many imitation learning settings.

\subsection{GAN-Based Binary Imitability}

Given an identification formula for the causal effect of a target variable $X$ on $S$, with $X$ conditioned on a mediator $W$, there is a separate system of equations for each instantiation of the set $W$, akin to the one shown in \cref{eq:frontdoor_alpha}. This means that the number of systems of equations to solve is exponential in the size of $W$, and can't easily be approached when the model contains continuous or high-dimensional variables.

To avoid reliance on these adjustment formulae, we solve for imitating policies through direct modeling of the SCM \cite{louizosCausalEffectInference2017,goudetCausalGenerativeNeural2017}. In particular, we follow a similar procedure to \cite{kocaogluCausalGANLearningCausal2017,goudetLearningFunctionalCausal2017}, by training a generative adversarial network (GAN) to imitate the observational distribution, with a separate generator for each observable variable in the causal graph.

The advantage of this approach is that once a model is trained that faithfully reproduces the observational distribution on a given causal graph, any identifiable quantity will be identical in the trained model as in the original distribution, no matter the form of the underlying mechanisms and latent variables \cite[Definition 3.2.4]{pearl:2k}. This means that so long as you have an instrument for $X$, one can use the trained generator to optimize for a conditional policy by directly implementing it in the model, as shown in \Cref{lem:surrogate}.

In this experiment, we show that such an approach is, in fact, practical. To maintain the ability to compare our results with the ground truth, we use binary variables for each node in a tested graph. This restriction is loosened in the final two experiments - our goal here is to show that it is possible to get very accurate model reconstruction and imitation (with accurate intervention effects) with a GAN, even when there are latent variables.

Like in the previous experiment, for ``ground-truth" data, our models are sampled uniformly from the space of factorized distributions. For example, for graph 1 in \Cref{table:experiment}, one can factorize $P(x,y,z)=P(x)P(z|x)P(y|x,z)$, and can choose ground-truth probabilities $P(x),P(z|x'),P(y|x', z')\sim U(0,1), \forall x, y, z, x',z'$.

For the GANs, we adapt discrete BGAN \cite{hjelmBoundarySeekingGenerativeAdversarial2018} to arbitrary SCM. The advantage of f-divergence-based approaches (such as BGAN) is that the f-GAN discriminators explicitly optimize a lower bound on the value of the chosen f-divergence. In particular, defining an f-divergence over two distributions $P$ and $Q$ as $D_f(P||Q)=\int_X q(x) f\left(\frac{p(x)}{q(x)}\right)dx$\footnote{The KL divergence is an example of an f-divergence with $f(x)=x\log(x)$}, with $f^*$ as the convex conjugate of $f$, and using $\mathcal{T}$ as the set of functions that a neural network can implement, the f-GAN discriminator optimizes the following \cite{nowozinFGANTrainingGenerative2016,nguyenEstimatingDivergenceFunctionals2010}:
\[
D_f(P||Q) \ge \sup_{T\in \mathcal{T}}\left(\E_{x\sim P}[T(x)] - \E_{x\sim Q}[f^*(T(x))]\right)
\]
with a tight bound for $T^*(x)=f'\left(\frac{p(x)}{q(x)}\right)$. This means that by choosing an appropriate function $f$ and the class of functions $T$, one can estimate the divergence through optimization over $T$ \cite{nguyenEstimatingDivergenceFunctionals2010}. With this trained discriminator $T(x)$, one can update generator $Q_\theta$ to be more similar to $P$ with a update similar to a policy-gradient \cite{hjelmBoundarySeekingGenerativeAdversarial2018}.

To adapt these GANs to work with causal models, we create a separate generator for each variable $V \in \*V$, which is given as input samples from $\Pa_{V}$, and outputs multinomial probabilities. Each latent variable is explicitly sampled from $\mathcal{N}(0,1)^k$ (with $k=3$ in these experiments), and given as input to its children. Finally, instead of a single discriminator for the entire model, we exploit independence relations in the graph to localize optimization. In particular, in a Markovian causal model (without latent variables), one can create a separate discriminator for each variable, to directly learn the conditional distribution for each node. However, once there are latent variables, this variable-focused approach no longer captures all dependencies. Instead, we construct a separate conditional discriminator for each set of nodes that have a path made up of entirely bi-directed edges (c-components \cite{tian:pea03-r290-A}), which allows each discriminator to specialize to a part of the model.

As an example, for the graph 1 in \cref{table:experiment}, the generator for $X$ gets as input a sample $u \sim \mathcal{N}(0,1)^3$, and outputs a probability of $1$. The generator for $Z$ gets as input a sample according to the probability of $X$, and outputs a conditional probability of $Z$. Finally, $Y$ gets as input $u$ and the sampled value of $Z$. This graph has 2 c-components ($\set{X, Y}$ and $\set{Z}$), which corresponds to discriminators $P(z|x)$ and $P(x,y|z)$.

Similarly, once the model is optimized, the policy $\pi$ is trained by manually replacing the generator for $X$ with a new, untrained generator for $\pi$, and training it as a GAN with a discriminator comparing samples of $Y$ in the original model (i.e., $Y \sim P(y)$) with samples of $Y$ in the intervened model (i.e., $Y \sim P(y|\doo(\pi); \hat M)$ where $\hat M$ is the parametrized model learned by generators.)

The graphs in \cref{table:experiment} were each chosen to demonstrate a different aspect of imitability. The first graph is imitable by direct parents (\Cref{thm:dp}), and corresponds to existing approaches to imitation, where an agent gets observations identical to the expert. The second graph demonstrates an example where an expert has additional information from a latent variable, but the agent can still successfully imitate $P(y)$ by using information that is not used by the expert, i.e., the $\pi$-backdoor admissible set $\set{W}$ (\Cref{thm:backdoor}). Finally, Graphs 3 and 4 demonstrate data-dependent imitability (\Cref{def:p-imtability}). Graph 3 focuses on distributions that are amenable to imitation, meaning that imitation of $P(y)$ is possible without knowledge of the latent variable. Graph 4 uses only non-imitable distributions. Sampled uniformly, half of these instances are p-imitable, and half are not, so the expected performance of an unknown distribution is the average of Row 3 and 4.

The table columns show the average L1 distance between the computed interventional distributions, the predicted \& optimal policy, as well as the effect $P(y|\doo(\pi))$ of the learned policy vs the observed $P(y)$. Each element of the table shows an average value of its corresponding distance from ground-truth over 100 runs (with each run sampling a different ground-truth probability distribution over the graph), overlaid over the histogram of distances over the 100 distributions/trained GANs.

\begin{table}[t]

\centering
\begin{tabular}[t]{m{.6cm}@{\hspace*{2mm}}>{\centering\arraybackslash}m{2.01cm} |  >{\centering\arraybackslash}m{2.21cm}@{\hspace*{1mm}} >{\centering\arraybackslash}m{2.21cm}@{\hspace*{1mm}} >{\centering\arraybackslash}m{2.21cm}@{\hspace*{1mm}} >{\centering\arraybackslash}m{2.21cm}}
\toprule
\# & Graph & $|y_{x_0} -\hat y_{x_0}|$ & $|y_{x_1}-\hat y_{x_1}|$& $|\alpha-\hat\alpha|$ &  $|y-\hat y|$\\
\midrule
1 & \begin{minipage}{2cm}
\begin{adjustbox}{max totalsize={.99\textwidth}{.8\textheight},center}
\begin{tikzpicture}
      \def\outerr{3}
      \def\innerr{2.7}
      \node[vertex] (X) at (0, 0) {$Z$};
      \node[vertex] (Z) at (1, 0) {$X$};
      \node[vertex] (Y) at (2, 0) {$Y$};

      \draw[dir] (X) -- (Z);
      \draw[dir] (Z) -- (Y);

      \draw[bidir] (X) to [bend left=45] (Y);
      
      \begin{pgfonlayer}{back}
        \draw[fill=betterred!25,draw=none] \convexpath{Z, X}{\outerr mm};
        \node[circle,fill=betterred!65,draw=none,minimum size=2*\innerr mm] at (Z) {};
        \node[circle,fill=betterblue!25,draw=none,minimum size=2*\outerr mm] at (Y) {};
        \node[circle,fill=betterblue!65,draw=none,minimum size=2*\innerr mm] at (Y) {};
    \end{pgfonlayer}
  \end{tikzpicture}
  \end{adjustbox}
  \end{minipage}
  &\includegraphics[width=2.2cm]{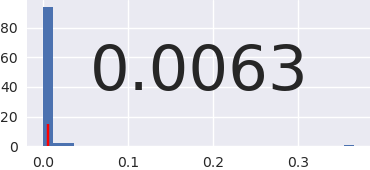}&
  \includegraphics[width=2.2cm]{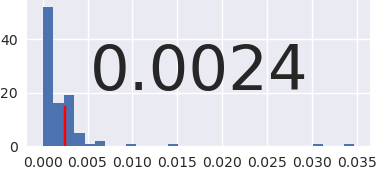} & \includegraphics[width=2.2cm]{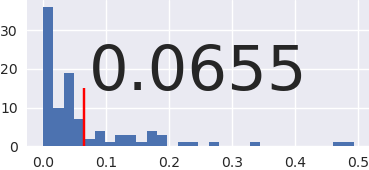} & \includegraphics[width=2.2cm]{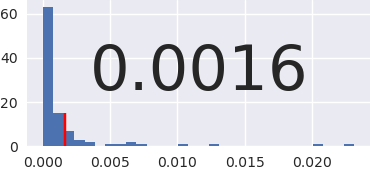}\\

2 & \begin{minipage}{2cm}
\begin{adjustbox}{max totalsize={.99\textwidth}{.8\textheight},center}
\begin{tikzpicture}
    \def\outerr{3}
      \def\innerr{2.7}
      \node[vertex] (X) at (0, 0) {$W$};
      \node[vertex] (Z) at (1, 0) {$Z$};
      \node[vertex] (W) at (2, 0) {$X$};
      \node[vertex] (Y) at (3,0) {$Y$};

      \draw[dir] (X) -- (Z);
      \draw[dir] (Z) -- (W);
      \draw[dir] (W) -- (Y);

      \draw[bidir] (X) to [bend left=45] (Y);
      \draw[bidir] (Z) to [bend left=45] (W);

      \begin{pgfonlayer}{back}
        \draw[fill=betterred!25,draw=none] \convexpath{W, Z, X}{\outerr mm};
        \node[circle,fill=betterred!65,draw=none,minimum size=2*\innerr mm] at (W) {};
        \node[circle,fill=betterblue!25,draw=none,minimum size=2*\outerr mm] at (Y) {};
        \node[circle,fill=betterblue!65,draw=none,minimum size=2*\innerr mm] at (Y) {};
    \end{pgfonlayer}
  \end{tikzpicture}
  \end{adjustbox}
  \end{minipage}
  &\includegraphics[width=2.2cm]{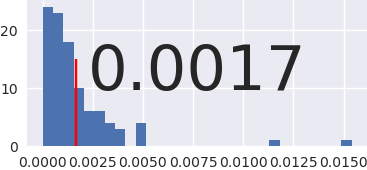}&
  \includegraphics[width=2.2cm]{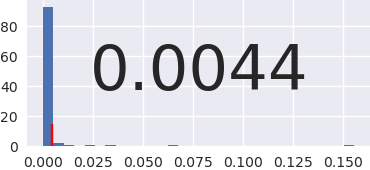} & \includegraphics[width=2.2cm]{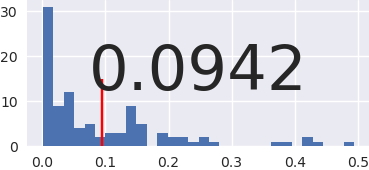} & \includegraphics[width=2.2cm]{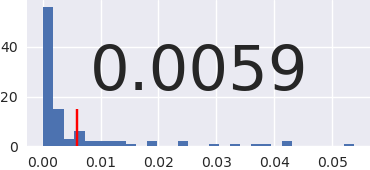}\\

3 & \begin{minipage}{2cm}
\begin{adjustbox}{max totalsize={.99\textwidth}{.8\textheight},center}
\begin{tikzpicture}
\def\outerr{3}
      \def\innerr{2.7}
      \node[vertex] (X) at (0, 0) {$X$};
      \node[vertex] (Z) at (1, 0) {$W$};
      \node[vertex] (Y) at (2, 0) {$Y$};

      \draw[dir] (X) -- (Z);
      \draw[dir] (Z) -- (Y);

      \draw[bidir] (X) to [bend left=45] (Y);
      
      \begin{pgfonlayer}{back}
      \node[circle,fill=betterred!25,draw=none,minimum size=2*\outerr mm] at (X) {};
        \node[circle,fill=betterred!65,draw=none,minimum size=2*\innerr mm] at (X) {};
        \node[circle,fill=betterblue!25,draw=none,minimum size=2*\outerr mm] at (Y) {};
        \node[circle,fill=betterblue!65,draw=none,minimum size=2*\innerr mm] at (Y) {};
    \end{pgfonlayer}
  \end{tikzpicture}
  \end{adjustbox}
  \end{minipage}
  &\includegraphics[width=2.2cm]{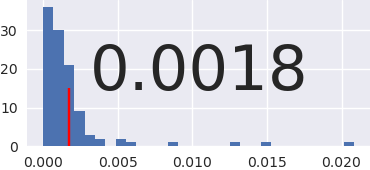}& \includegraphics[width=2.2cm]{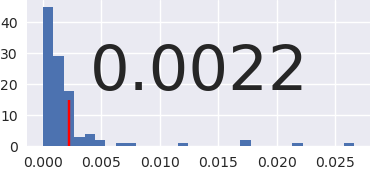} & \includegraphics[width=2.2cm]{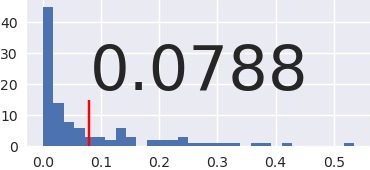} & \includegraphics[width=2.2cm]{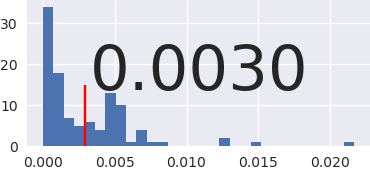}\\

4 & \begin{minipage}{2cm}
\begin{adjustbox}{max totalsize={.99\textwidth}{.8\textheight},center}
\begin{tikzpicture}
\def\outerr{3}
      \def\innerr{2.7}
      \node[vertex] (X) at (0, 0) {$X$};
      \node[vertex] (Z) at (1, 0) {$W$};
      \node[vertex] (Y) at (2, 0) {$Y$};

      \draw[dir] (X) -- (Z);
      \draw[dir] (Z) -- (Y);

      \draw[bidir] (X) to [bend left=45] (Y);
      
      \begin{pgfonlayer}{back}
      \node[circle,fill=betterred!25,draw=none,minimum size=2*\outerr mm] at (X) {};
        \node[circle,fill=betterred!65,draw=none,minimum size=2*\innerr mm] at (X) {};
        \node[circle,fill=betterblue!25,draw=none,minimum size=2*\outerr mm] at (Y) {};
        \node[circle,fill=betterblue!65,draw=none,minimum size=2*\innerr mm] at (Y) {};
    \end{pgfonlayer}
  \end{tikzpicture}
  \end{adjustbox}
  \end{minipage}
  &\includegraphics[width=2.2cm]{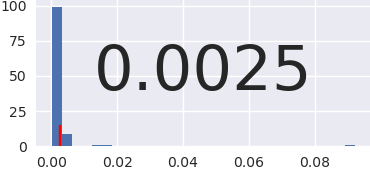}&
  \includegraphics[width=2.2cm]{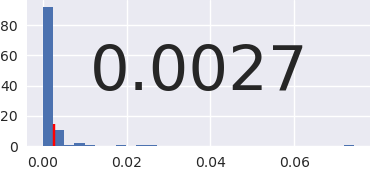} & \includegraphics[width=2.2cm]{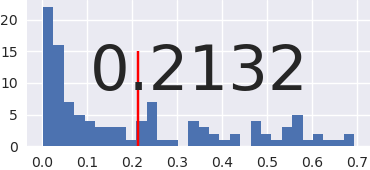} & \includegraphics[width=2.2cm]{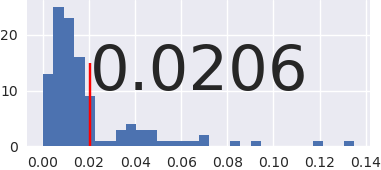}\\
\bottomrule
\\
\end{tabular}

\caption{Results of imitation using GANs in randomly sampled distributions. Each number is averaged over 100 randomly sampled models. Experiments in Row 3 and 4 are run with p-imitable, and not p-imitable distributions, respectively.}
\label{table:experiment}
\end{table}

The first two columns allow determining the error in reconstructing the interventional distribution of atomic intervention $\doo(x)$. The $|y_x-\hat y_x|$ represents the differences between the ground-truth $E[y|\doo(x)]$ and the imitated $E[y | \doo(x); \hat{M}]$ in the parametrized model $\hat M$. Notice that the policies are often conditioned, which is not reflected in these values.

The $|y-\hat y|$ value represents the difference in the ground-truth expert's reward $E[y]$ with the distribution $E[y| \doo(\pi)]$ induced by the learned policy $\pi$ \textit{in the ground-truth model}, meaning that the policy is trained in the imitated model, but is tested by replacing the true mechanism, as if the learned mechanism was tried in real life.

The main point of possible confusion could be $|\alpha-\hat\alpha|$ in the graphs where the policy is conditional. In the frontdoor cases, when the policy has no conditioning (Rows 3 and 4), the optimal value can be cleanly found, using \cref{eq:frontdoor_alpha}. However, in the backdoor case (as seen in Row 1 and in Row 2), there can be multiple possible valid solutions. Here we show the precise procedure used to compute $|\alpha-\hat\alpha|$ in these two cases.

\begin{itemize}
    \item In Row 1, which is commonly called the ``backdoor" graph, we have
defining $\alpha_1=\pi(x_1|z_1)$ and $\alpha_0=\pi(x_1|z_0)$:
\begin{equation}
\label{eq:backdoor_alpha}
\begin{aligned}
P(y) &= \sum_{x, z}P(z)\pi(x|z)P(y|x, z)\\
&= P(z_1)\alpha_1P(y|x_1, z_1) + P(z_1)(1-\alpha_1)P(y|x_0, z_1)\\
&\ \ \ + P(z_0)\alpha_0 P(y|x_1, z_0) + P(z_0)(1-\alpha_0)P(y|x_0, z_0)\\
&= P(z_1)P(y|x_0, z_1)+P(z_0)P(y|x_0, z_0) \\
&\ \ \ + \alpha_1 P(z_1)(P(y|x_1, z_1)-P(y|x_0, z_1))\\
&\ \ \ + \alpha_0 P(z_0)(P(y|x_1, z_0)-P(y|x_0, z_0))
\end{aligned}
\end{equation}

This means that multiple possible $\alpha_0,\alpha_1$ satisfy the given constraint. In such situations, given $\hat\alpha_1$ and $\hat\alpha_0$ found by GAN, we find the ``closest correct comparison" by minimizing
$$
P(z_0)|\hat\alpha_0 - \alpha_0| + P(z_1)|\hat\alpha_1 - \alpha_1|
$$

subject to the constraint in \cref{eq:backdoor_alpha}, and with $\alpha_1,\alpha_0 \in [0,1]$. We then report this minimized value in the column labeled $|\alpha-\hat\alpha|$.
\item Row 2 has more complex relations, since the policy has as inputs both values from $Z$ and $W$. However, the values of $W$ are irrelevant to imitating $y$, since $W$ is in a different c-component of the graph than $Y$. This means that we can use the same approach as for Row 1, considering only $Z$ (and averaging the policy over $W$ values). 
\end{itemize}

\subsubsection{Discussion}

The results in \cref{table:experiment} suggest that GANs are capable of training accurate imitating policies in the presence of latent variables. Of particular note is the relatively large error in $\alpha$ can sometimes be present in the policies, despite the policies yielding very accurate samples of $y$. This happens when the imitable policy has $P(y|\doo(x_0))$ and $P(y|\doo(x_1))$ with very similar values, meaning that the policy has little effect on the probability of $Y$.

\subsection{Highway Driving}

\begin{figure}[t]
    \hfill
    \begin{subfigure}{0.25\linewidth}
    \centering
\begin{tikzpicture}
      \def\outerr{3}
      \def\innerr{2.7}
      \node[vertex] (X) at (0, 0) {$X$};
      \node[vertex,opacity=0] (H1) at (0, -0.5) {$H$};
      \node[vertex,opacity=0] (H2) at (0, 2.1) {$H$};
      \node[uvertex] (L) at (2, 0.75) {$L$};
      \node[uvertex] (Y) at (3, 0) {$Y$};
      \node[vertex] (W) at (1, 0.75) {$W$};
      
      \draw[dir] (L) to [bend left=45] (X);
      \draw[dir] (X) -- (Y);
      \draw[dir] (L) -- (W);

      \draw[bidir] (W) to [bend right=45] (Y);

      \begin{pgfonlayer}{back}
        \draw[fill=betterred!25,draw=none] \convexpath{W, X}{\outerr mm};
        \node[circle,fill=betterred!65,draw=none,minimum size=2*\innerr mm] at (X) {};
        \node[circle,fill=betterblue!25,draw=none,minimum size=2*\outerr mm] at (Y) {};
        \node[circle,fill=betterblue!65,draw=none,minimum size=2*\innerr mm] at (Y) {};
    \end{pgfonlayer}
    \end{tikzpicture}
    \caption{}
    \label{fig:backdoor_binary}
    \end{subfigure}\hfill
    \begin{subfigure}{0.25\linewidth}\centering
\begin{tikzpicture}
      \def\outerr{3}
      \def\innerr{2.7}
      \node[vertex] (X) at (0, 0) {$X$};
      \node[vertex] (Z) at (1.5, 2) {$Z$};
      \node[vertex,opacity=0] (H1) at (0, -0.5) {$H$};
      \node[vertex,opacity=0] (H2) at (0, 2.1) {$H$};
      \node[uvertex] (L) at (2, 0.75) {$L$};
      \node[uvertex] (Y) at (3, 0) {$Y$};
      \node[vertex] (W) at (1, 0.75) {$W$};
      
      \draw[dir] (L) to [bend left=45] (X);
      \draw[dir] (Z) -- (Y);
      \draw[dir] (X) -- (Y);
      \draw[dir] (L) -- (W);
      \draw[dir] (Z) to (X);
  
      \draw[bidir] (W) to [bend right=45] (Y);

      \begin{pgfonlayer}{back}
        \draw[fill=betterred!25,draw=none] \convexpath{Z, W, X}{\outerr mm};
        \node[circle,fill=betterred!65,draw=none,minimum size=2*\innerr mm] at (X) {};
        \node[circle,fill=betterblue!25,draw=none,minimum size=2*\outerr mm] at (Y) {};
        \node[circle,fill=betterblue!65,draw=none,minimum size=2*\innerr mm] at (Y) {};
    \end{pgfonlayer}
    \end{tikzpicture}
    \caption{}
    \label{fig:backdoor_highd}
    \end{subfigure}\hfill\null
    \caption{The graphs used for the backdoor experiments. First, \cref{fig:backdoor_binary} was optimized adversarially with binary variables to maximize na\"ive imitation error, then continuous data from the highD dataset were added in \cref{fig:backdoor_highd}, while maintaining the adversarial imitation error.}
\end{figure}
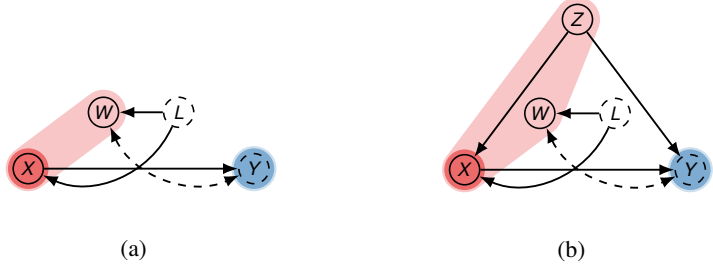

The purpose of this experiment is to demonstrate that when imitating $P(y)$, it is important to choose a set of covariates that is $\pi$-backdoor admissible (\Cref{def:backdoor}). To witness, in \cref{fig:backdoor_binary}, using knowledge of $W$ to imitate $X$ effectively adds confounding between $X$ and $Y$, possibly affecting the imitated distribution of $Y$. That is, $\set{W}$ is not $\pi$-backdoor admissible. Similarly, in \cref{fig:backdoor_highd}, one must use knowledge of $Z$ when imitating $P(y)$, but using either only $W$ or both $Z$ and $W$ can lead to bias and inferior performance on reward measure $Y$. In other words, set $\set{Z}$ is $\pi$-backdoor admissible while $\set{W, Z}$ is not due to the active path $X \leftarrow L \rightarrow W \leftrightarrow Y$.

We demonstrate this in a two-step procedure. First, we performed an automated adversarial search over the binary distributions consistent with \cref{fig:backdoor_binary}, to \textit{maximize} $\E[Y]$ when imitating just $\pi(x) = P(x)$, but \textit{minimizing} $\E[Y]$ when exploiting knowledge of $W$ to imitate $\pi(x|w)=P(x|w)$. This search led to the following mechanisms ($\oplus$ is xor):

\[
P(l_1)=P(u_1)= \frac{\sqrt{5}-1}{2}\approx 0.62 \ \ \ \ \ W \gets L\land U \ \ \ \ X \gets L \ \ \ \ Y \gets X \oplus U
\]

This means that under ``correct" imitation (not using knowledge of $W$), we have:

\[
\E[Y|\doo(\pi_1)] = \sum_{x, u} \E[Y|x, u]\pi_1(x)P(u) = 2\sqrt{5}-4 \approx  0.472
\]

Under ``incorrect" imitation (using $W$) we get:
\[
\pi_2(x_1|w) = P(x_1|w) = \frac{\sum_{l, u}P(u)P(l)P(x_1|l)P(w|l, u)}{\sum_{l, u}P(l)P(u)P(w|l, u)}= \left\{ \begin{aligned}
 2\frac{\sqrt{5} - 2}{\sqrt{5} - 1} \approx 0.382 & \text{ if $W=0$}\\
1 & \text{ if $W=1$}
\end{aligned}\right \}
\]

This means:
\[
\E[Y| \doo(\pi_2)] = \sum_{x, l, u, w}\E[Y|x, u]\pi_2(x|w)P(w|l, u)P(l)P(u) =4\frac{4\sqrt{5} - 9}{\sqrt{5} - 3} \approx 0.2918
\]

Therefore, using the mechanisms above in \cref{fig:backdoor_binary} with binary variables, using $W$ leads to a bias of $18\%$:

\[
\E[Y| \doo(\pi_1)]-\E[Y | \doo(\pi_2)] =  2\sqrt{5}-4 - 4\frac{4\sqrt{5} - 9}{\sqrt{5} - 3} \approx 0.18
\]

Indeed, by sampling 10,000 data points, and using the empirical $P(x)$ for one ``imitator", and $P(x|w)$ for the other, we get $0.1793$ difference between the two, corroborating this result.

\subsubsection{Making Things Continuous}

We next show that this same issue can show up when variables are continuous. To achieve this, we adapted car velocity data from the \texttt{highD} dataset to a model similar to the binary version. In this model, one must use $Z$ (velocity of the front car) to predict $X$ (velocity of the driving car), but $W$ would bias the prediction (with the same error as in previous section). The full model specification is as follows:
\begin{enumerate}
    \item $P(L = 1)=P(U = 1)=0.62$ (with values of $L$ constructed from values of $X$)
    \item $W\gets L\land U$
    \item $Z$ is velocity of preceding car from \texttt{highD} dataset.
    \item $X$ is velocity of current car from \texttt{highD} dataset. $L$ was constructed such that $X$ satisfies the relation $L = I_{\set{X-Z > -0.4}}$ (this was achieved by choosing the threshold $-0.4$ to give the correct distribution over $L$).
    \item $Y \gets U \land I_{\set{X-Z\le -0.4}} \vee \lnot U \land I_{\set{X-Z>-0.4}}$
\end{enumerate}
The values and mechanisms here were specifically chosen to have similar outputs to the previous model (\cref{fig:backdoor_binary}), despite using continuous car velocity data in place of boolean for node values. \cref{fig:backdoor_highd} shows a faithful graphical representation of the model. However, during the experiment, all algorithms are provided with only the causal diagram in \Cref{fig3a}. That is, only independence relationships encoded in \Cref{fig3a} are exploited.

\begin{figure}[ht]
    \hfill
    \begin{subfigure}{0.32\linewidth}
    \centering
    \includegraphics[width=\linewidth]{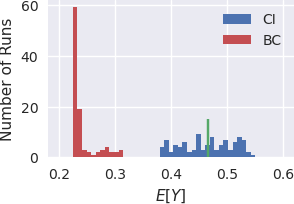}
    \caption{}
    \label{fig:backdoor_supervised}
    \end{subfigure}\hfill
    \begin{subfigure}{0.32\linewidth}\centering
    \includegraphics[width=\linewidth]{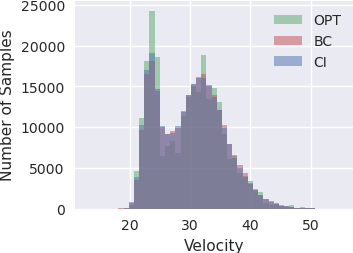}
    \caption{}
    \label{fig:backdoor_gan_dist}
    \end{subfigure}\hfill
    \begin{subfigure}{0.32\linewidth}\centering
    \includegraphics[width=\linewidth]{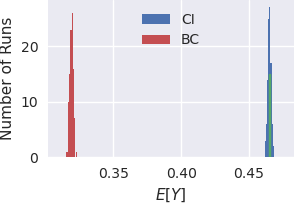}
    \caption{}
    \label{fig:backdoor_gan}
    \end{subfigure}\hfill\null
    \caption{Results of experiments using the continuous ``backdoor" model. (a) and (c) are histograms of 100 independent runs of a direct supervised learning and GAN approach, respectively. (b) shows the distribution over X of a trained GAN model}
\end{figure}

We performed two separate imitation experiments, both of which show that (1) using the covariates set $\set{W, Z}$ introduces biases into the imitator's reward $P(y|\doo(\pi))$; (2) using the $\pi$-backdoor admissible set $\set{Z}$ allows one to successfully imitate the expert's performance $P(y)$. The first experiment was imitation using standard supervised learning. Two separate (2 layer) neural networks were trained on a small subset of data using L2 loss over X. The first had as its input the values of $Z$ (velocity of car ahead), while the second had as input values of $W$ and $Z$. The trained outputs of these networks were used for imitation using the full dataset, and the resulting policy was evaluated by its induced $\E[Y | \doo(\pi)]$ (Y represents an unknown reward, so bigger is better). This experiment was repeated 100 times, giving a distribution over performance of trained models, shown in \cref{fig:backdoor_supervised}. Note that while the performance varied widely, knowledge of $W$ had a clear negative effect. 

The outputs of a supervised learning algorithm are single values - they do not represent a distribution over possible values. This can affect imitation, since it is possible that the full range of the distribution is necessary for optimal imitation. We next trained two GANs to imitate the distribution over X, in the same way as done with the supervised models - one using only $Z$, and one using both $W$ and $Z$. Due to the sampling needs of GANs, however, they had access to the full dataset over $X,Z,W$. This experiment was also repeated 100 times, giving the performance shown in \cref{fig:backdoor_gan}.
Once again, the GAN using only $Z$ has clearly superior performance. This difference exists despite both trained GANs seemingly recovering a good approximation over the distribution of $X$, as seen in \cref{fig:backdoor_gan_dist}.

\subsection{MNIST Digits}
The purpose of this experiment is to show how Algorithm \Cref{imitate} using GANs for policy estimation could obtain reasonable imitation results in the p-imitable setting, even when the probability distribution of some observed endogenous variables is high-dimensional.

Specifically, we use the frontdoor graph, same as in the first experiment. There, when $X$ and $S$ are binary, we can compute the value $\alpha$ using \cref{eq:backdoor_alpha}. However, in this case, the formula for the interventional distribution is \cite{pearl:09b}:
\begin{figure}[H]
\begin{subfigure}{0.3\linewidth}\centering
\begin{tikzpicture}
    \def\outerr{3}
    \def\innerr{2.7}
    \node[vertex] (X) at (0, 0) {$X$};
      \node[vertex] (W) at (0, -1) {$W$};
      \node[vertex] (S) at (1, -1) {$S$};
      \node[uvertex] (Y) at (2, -1) {$Y$};
      
      \draw[dir] (X) -- (W);
      \draw[dir] (W) -- (S);
      \draw[dir] (S) -- (Y);

      \draw[bidir] (X) to [bend left=30] (S);

      \begin{pgfonlayer}{back}
        \draw[fill=betterblue!25,draw=none] \convexpath{S, Y}{\outerr mm};
        \node[circle,fill=betterred!25,draw=none,minimum size=2*\outerr mm] at (X) {};
        \node[circle,fill=betterred!65,draw=none,minimum size=2*\innerr mm] at (X) {};
        \node[circle,fill=betterblue!25,draw=none,minimum size=2*\outerr mm] at (Y) {};
        \node[circle,fill=betterblue!65,draw=none,minimum size=2*\innerr mm] at (Y) {};
    \end{pgfonlayer}
  \end{tikzpicture}
\end{subfigure}
\begin{subfigure}{.65\linewidth}
 \begin{equation}
\label{eq:frontdoor_do}
P(s|\doo(\pi)) = \sum_{x} \pi(x) \sum_w P(w|x)\sum_{x'}P(s|x',w)P(x')
\end{equation}
\end{subfigure}
\end{figure}

Critically, computing the effect of an intervention here requires a sum (or integral if continuous) over $W$. If $W$ is a complex distribution or is high-dimensional, this quantity can be difficult to estimate.

Once again, to allow us the ability to compare results to a ground-truth value, we repeat the procedure performed in the previous experiment, by first constructing a binary model with known characteristics, and then by replacing the binary value of $W$ with a high-dimensional distribution with a property that follows the same underlying mechanism as the original binary value. Specifically, we use the \texttt{MNIST} digits for $0$ and $1$, which are a $28\times 28=784$ dimensional vector representing a complex probability distribution. The pictures of 0s replace ``\textsc{False}'' values of $W$, while pictures of $1$s replace ``\textsc{True}'' values. This allows a direct translation between a binary variable and the desired complex distribution. The underlying binary distribution had the following mechanisms:
$$
\begin{aligned}
P(u_1)&=0.9\\
P(x_1|u_0)=0.1\ &\ \ P(x_1|u_1)=0.9\\
P(w_1|x_1)=0.1\ &\ \ P(w_1|x_1)=0.9\\
P(s_1|u_0, w_0) = 0.1\ &\ \ P(s_1|u_0, w_1)=0.9\\
P(s_1|u_1, w_0) = 0.9\ &\ \ P(s_1|u_1, w_1)=0.1\\
\end{aligned}
$$
This leads to $P(s_1)=0.245$, but with na\"ive cloning of $P(x)$, the resulting $P(s_1 | \doo(\pi))=0.334$, a difference of $0.0888$. We chose the reward signal $Y\gets \lnot S$, since the assumption is that the original values of $X$ were decided by an expert.

The GAN for the c-component $\set{W}$ is much larger than that of the other variables, since $W$ is an image, so we pre-trained this component, and inserted the trained version into the full graph for optimization. We repeated experiment 16 times, of which 3 runs were discarded due to collapse in the GANs associated with $P(w|x)$. The results are visible in \cref{fig:mnistgan_results}. These results show that despite the high-dimensional nature of the distribution over $W$, the GAN was consistently able to perform better than na\"ive imitation, approaching the optimal value.

\begin{figure}[t]
\begin{subfigure}{0.32\linewidth}
 \centering
    \includegraphics[width=0.9\linewidth]{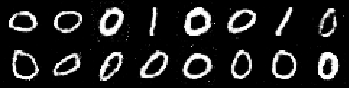}
    \caption{}
\end{subfigure}
\begin{subfigure}{0.32\linewidth}
     \centering
    \includegraphics[width=0.9\linewidth]{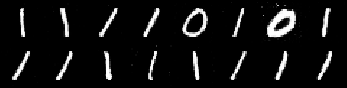}
    \caption{}
\end{subfigure}
\begin{subfigure}{0.32\linewidth}
 \centering
    \includegraphics[width=0.99\linewidth]{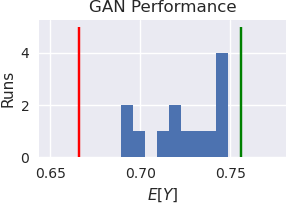}
    \caption{}
\end{subfigure}
    
    \caption{In (a) and (b), results of a trained GAN conditional on $X= 0$ and $X=1$ respectively ($0$ for \textsc{False}, $1$ for \textsc{True}). In (c) is shown a histogram of the values obtained over multiple runs of the GAN. The red line represents na\"ive imitation, and the green line represents optimal imitation.}
    \label{fig:mnistgan_results}
\end{figure}

\end{document}